\documentclass{article}

\usepackage{geometry}
\usepackage{fancyhdr, titling}

\newcommand\arxivversion

\usepackage[nottoc]{tocbibind}
\PassOptionsToPackage{backref=page}{hyperref}

\usepackage{common}

\title{A unified framework for bandit multiple testing}
\author{Ziyu Xu$^1$, Ruodu Wang$^3$, Aaditya Ramdas$^{1, 2}$\\
\and
Departments of $^1$Statistics and $^2$Machine Learning, Carnegie Mellon University\\
$^3$Department of Statistics and Actuarial Science, University of Waterloo\\
\and
\texttt{\{xzy,aramdas\}@cmu.edu, wang@uwaterloo.ca}}

\date{\today}

\begin{document}

\maketitle
\begin{abstract}
In bandit multiple hypothesis testing, each arm corresponds to a different null hypothesis that we wish to test, and the goal is to design adaptive algorithms that correctly identify large set of interesting arms (true discoveries), while only mistakenly identifying a few uninteresting ones (false discoveries). One common metric in non-bandit multiple testing is the false discovery rate (FDR). We propose a unified, modular framework for bandit FDR control that emphasizes the decoupling of exploration and summarization of evidence. We utilize the powerful martingale-based concept of ``e-processes'' to ensure FDR control for arbitrary composite nulls, exploration rules and stopping times in generic problem settings. In particular, valid FDR control holds even if the reward distributions of the arms could be dependent, multiple arms may be queried simultaneously, and multiple (cooperating or competing) agents may be querying arms, covering combinatorial semi-bandit type settings as well. Prior work has considered in great detail the setting where each arm's reward distribution is independent and sub-Gaussian, and a single arm is queried at each step. Our framework recovers matching sample complexity guarantees in this special case, and performs comparably or better in practice. For other settings, sample complexities will depend on the finer details of the problem (composite nulls being tested, exploration algorithm, data dependence structure, stopping rule) and we do not explore these; our contribution is to show that the FDR guarantee is clean and entirely agnostic to these details. 

\end{abstract}
\tableofcontents

\section{Introduction to bandit multiple hypothesis testing}
\label{sec:Intro}
\ifarxiv{}{\vspace{-5pt}}

Scientific experimentation is often a sequential process. To test a single  null hypothesis --- with ``null'' capturing the setting of no scientific interest, and the alternative being scientifically interesting --- scientists typically collect an increasing amount of experimental data in order to gather sufficient evidence such that they can potentially reject the null hypothesis (i.e.\ make a scientific discovery) with a high degree of statistical confidence. As long as the collected evidence remains thin, they do not reject the null hypothesis and do not proclaim a discovery. 
Since executing each additional unit of data (stemming from an experiment or trial) has an associated cost (in the form of time, money, resources), the scientist would like to stop as soon as possible. This becomes increasingly prevalent when the scientist is testing multiple hypotheses at the same time, and investing resources into testing one means divesting it from another.

For example, consider the case of a scientist at a pharmaceutical company who wants to discover which of several drug candidates under consideration are truly effective (i.e.\ testing a hypothesis of whether each candidate has greater than baseline effect) through an adaptive sequential assignment of drug candidates to participants.
Performing follow up studies on each discovery is expensive, so the scientist does not want to make many ``false discoveries'' i.e.\ drugs that did not have an actual effect, but were proclaimed to have one by the scientist.
To achieve these goals, one could imagine the scientist collecting more data for candidates whose efficacy is unclear but appear promising (e.g.\ drugs with nontrivial but inconclusive evidence), and stop sampling candidates that have relatively clear results already (e.g.\ drugs that have a clear and large effect, or seemingly no effect). 

\ifarxiv{\paragraph}{\textbf}{Past work.} This problem combines the challenges of multiple hypothesis testing with multi-arm bandits (MABs). In a ``doubly-sequential'' version of the problem studied by~\citet{yang2017framework}, one encounters a sequence of MAB problems over time. Each MAB was used to test a single special placebo arm against several treatment arms, and if at least one treatment dominated the placebo, then they aimed to return the best treatment. Thus each MAB was itself a single adaptive sequential  hypothesis test, 
and the authors aimed not to make too many false discoveries over the sequence of MAB instances.

This paper instead considers the  formulation of
 \citet{jamieson_bandit_2018}, henceforth called \JJ, 
but our techniques apply equally well to the above setup.
\emph{To avoid confusions, note that our setup is very different from the active classification work of the same authors~\citep{jain2020new}.}
 To recap, \JJ~consider a single MAB instance without a placebo arm (or rather, leaving it implicit),
 and try to identify as many treatments that work better than chance as possible, without too many false identifications.
 To clarify, we associate each arm with one (potentially composite) null hypothesis --- for example, the hypothesis that corresponding drug has no (significant) effect. A single observed reward when pulling an arm corresponds to a statistic that summarizes the results of one experiment with the corresponding drug, and the average reward across many experiments could correspond to an estimate of the average treatment effect, which would be (at most) zero for null arms and positive for non-nulls. 
Thus, a strategy for quickly finding the arms with positive means corresponds to a strategy for allocating trial patients to drug candidates that allows the scientists to rapidly find the effective drugs.

However, the above corresponds to only the simplest problem setting. In more complex settings, it may be possible to pull multiple arms in each round, and observe correlated rewards. Further, the arms may have some combinatorial structure that allows only certain subsets of arms to be pulled. There could be multiple agents (eg: hospitals) pulling the same set of arms and seeing independent rewards (eg: different patients) or dependent rewards (eg: patient overlap or interference). Further, if some set of experiments by one scientist yielded suggestive but inconclusive evidence, another may want to follow up, but not start from scratch, instead picking up from where the first left off. Last, the MAB may be stopped for a variety of reasons that may or may not be in the control of the scientist (eg: a faster usage of funding than expected, or additional funding is secured). We dive in the details of these scenarios in \Cref{subsec:MultipleAgents}.

\ifarxiv{\paragraph}{\textbf}{Our contribution.} We introduce a modular meta-algorithm for bandit multiple testing with provable \(\FDR\) control that utilizes ``e-values'' --- or, more appropriately, their sequential analog, ``e-processes'' --- a recently introduced alternative to p-values (or p-processes) by \citet{ramdas_testing_exchangeability_2021a} for various testing problems, that are inherently related to martingales, gambling and betting~\citep{shafer_testing_betting_2021,grunwald_safe_testing_2020,howard2020time,wang_false_2020}. This work is the first to carefully study e-processes in general MAB settings, building on prior work that studied a special case~\citep{wang_false_2020}. We also are the first to extend the bandit multiple testing problem to the combinatorial bandit setting --- \JJ\ had previously only analyzed the problem in the single-arm, independent reward setting. Utilizing e-processes provide our meta-algorithm with several benefits. 
(a) For composite nulls, it is typically easier to construct e-processes than p-processes; the same holds when data from a single source is dependent. When combining evidence from disparate (independent or dependent) sources, it is also more straightforward to combine e-values than p-values (see \Cref{subsec:MultipleAgents}).
(b) The same multiple testing step applies in all bandit multiple testing problems, regardless of all the various details of the problem setup mentioned in the previous paragraph. Consequently, \(\FDR\) control in our meta-algorithm is agnostic to much of problem setup and can be proved in a vast array of settings. This is not true when working for p-values.  In particular, the techniques for proving \(\FDR\) control in \JJ\ are highly reliant on the specific bandit setup in their paper.
(c) The exploration step can be --- but does not have to be --- decoupled from the multiple testing (combining evidence) step. This results in a modular procedure that can be easily ported to new problem settings to yield transparent guarantees on \(\FDR\) control.

By virtue of being a meta-algorithm, we do not (and cannot) provide ``generic'' sample complexity guarantees: these will depend on all of the finer problem details mentioned above, on the exploration algorithm employed, on which e-processes are constructed. Our emphasis is on the flexibility with which FDR control can be guaranteed in a vast variety of problem setups. Further research can pick up one problem at a time and design sensible exploration strategies and stopping rules, developing sampling complexity bounds for each, and these bounds will be inherited by the meta-algorithm. However, we do formulate some generic exploration algorithms in \Cref{sec:GenericAlgs} based on best arm identification algorithms \citep{audibert2010best,kalyanakrishnan_pac_subset_2012,chen2014combinatorial,jamieson_lil_ucb_2014,kaufmann_complexity_bestarm_2016,chen_nearly_optimal_2017a,jourdan_efficient_pure_2021a}.

When instantiated to the particular problem setup studied by \JJ\ (independent, sub-Gaussian rewards, one arm in each round, etc.), we get a slightly different algorithm from them --- the exploration strategy can be inherited to stay the same, but the multiple testing part differs. 
\JJ\ use p-processes for each arm to determine whether that arm should be added to the rejection set, and correct for testing multiple hypotheses by using the BH procedure \citep{benjamini_controlling_1995-2} to ensure that the false discovery rate (FDR), i.e.\ the proportion of rejections that are false discoveries in expectation, is controlled at some fixed level \(\delta\). Adaptive sampling induces a peculiar form of dependence amongst the p-values, for which the BH procedure provides error control at an inflated level;  in other words, one has to use BH at a more stringent level of approximately \(\delta / \log(16 / \delta)\) to ensure that the \(\FDR\) is less than \(\delta\). On the other hand, we use the e-BH procedure \citep{wang_false_2020}, an analogous procedure for e-values, which can ensure the \(\FDR\) is less than \(\delta\) without any inflation, regardless of the dependence structure between the e-values of each arm. Our algorithm has improved sample efficiency in simulations and the same sample complexity in theory.

\ifarxiv{\paragraph}{\textbf}{Formal problem setup.} We define the bandit as having \(k\) arms, and \(\dist_i\) as the (unknown) reward distribution for arm \(i \in [k] = \{1, \dots, k\}\).
Every arm $i$ is associated with a null hypothesis, which is represented by a known, prespecified set of distributions $\Pcal_i$. If $|\Pcal_i|=1$, it is a `point null hypothesis', and otherwise it is a `composite null hypothesis'. Examples of the latter include ``all $[0,1]$-bounded distributions with mean $\leq 0.5$'' or ``all $1$-sub-Gaussian distributions with mean $\leq 0$'' or ``all distributions that are symmetric around 0'' or ``all distributions with median $\leq 0$''. While we assume by default that all rewards from an arm are i.i.d., we also formulate tests for hypotheses on reward distributions that may violate this assumption in \Cref{sec:ACE}. If $\nu_i \in \Pcal_i$, then we say that the $i$-th null hypothesis is true and we call $i$ a null arm; else, we say $i$-th null hypothesis is false and we call it a non-null arm. Thus, the set of arms are partitioned into two disjoint sets: nulls \(\hypset_0 \subseteq [k]\) and non-nulls \(\hypset_1 \coloneqq [k] \setminus \hypset_0\). 

Let $\Kcal \subseteq 2^{[k]}$ denote the subsets of arms that can be  jointly queried in each round. 
At each time \(t\), the algorithm chooses a subset of arms \(\sampleset_t \in \Kcal\) to sample jointly from.   
The special choice of $\Kcal = \{\{1\},\{2\},\dots,\{k\}\}$ recovers the standard bandit setup, but otherwise this setting is known as combinatorial bandits with semi-bandit feedback \citep{chen_combinatorial_multiarmed_2016}. We also consider the special case of full-bandit feedback (the algorithm sees all rewards at each time step) in \Cref{subsec:StreamingData}. We denote the reward sampled at time \(t\) from arm \(i \in \sampleset_t\) as \(\reward_{i, t}\). Let \(T_i(t)\) denote the number of times arm \(i\) has been sampled by time \(t\), and \(t_i(j)\) be the time of the \(j\)th sample from arm \(i\).

We now define a canonical ``filtration'' for our bandit problem. A filtration $(\filtration_t)_{t \geq 0}$ is a series of nested sigma-algebras that encapsulates what information is known at time $t$. (We drop the subscript and just write $(\filtration_t)$ for brevity, and drop the parentheses when just referring to a single sigma-algebra at time $t$.) 
Define the \textit{canonical filtration} as follows for \(t \in \naturals\): \(\filtration_t \coloneqq \sigma\left(\privrv \cup \{(i,s, X_{i, j}): s \leq t, i \in \sampleset_s\}\right)\) and we let \(\filtration_0 \coloneqq \sigma(\privrv)\) where \(\privrv\) is  uniformly distributed on $[0,1]$ and its bits capture all private randomness used by the bandit algorithm that are independent of all observed rewards.
Let \((\lambda_t)\) be a sequence of random variables indexed by \(t \in \naturals\). \((\lambda_t)\) is said to be \textit{predictable} w.r.t.\  \((\filtration_t)\) if \(\lambda_t\) is measurable w.r.t.\ \(\filtration_{t - 1}\) i.e.\ \(\lambda_t\) is fully specified given the information in \(\filtration_{t - 1}\). An $\naturals$-valued random variable \(\tau\) is a stopping time (or stopping rule) w.r.t.\ to \((\filtration_t)\) if \(\{\tau = t\} \in \filtration_t\) --- in other words, at each time $t$, we know whether or not to stop collecting data.
Let $\Tcal$ denote the set of all possible stopping times/rules w.r.t. $(\filtration_t)$, potentially infinite. Technically, the algorithm must not just specify a strategy to select $\sampleset_t$, but also specify when sampling will stop. This is denoted by the stopping rule or stopping time \(\stoptime \in \Tcal\). 

Once the algorithm halts at some time $\tau$, it produces a rejection set \(\rejset_\tau \subseteq [k]\). We consider two metrics w.r.t.\ \(\rejset\): the \(\FDR\) as discussed prior, and true positive rate \((\TPR)\), which is the proportion of non-nulls that are discovered in expectation. These two metrics are defined as follows: 
\ifarxiv{}{\vspace{-5pt}}
\begin{align*}
    \FDR(\rejset_\tau) ~\coloneqq~ \ifarxiv{\expect\left[\frac{|\hypset_0 \cap \rejset_\tau|}{|\rejset_\tau| \vee 1}\right],  \qquad \TPR(\rejset_\tau) ~\coloneqq~ \expect\left[\frac{|\hypset_1 \cap \rejset_\tau|}{|\hypset_1|}\right]}{\expect\left[\tfrac{|\hypset_0 \cap \rejset_\tau|}{|\rejset_\tau| \vee 1}\right],  \qquad \TPR(\rejset_\tau) ~\coloneqq~ \expect\left[\tfrac{|\hypset_1 \cap \rejset_\tau|}{|\hypset_1|}\right]}.
\end{align*}
\ifarxiv{}{\vspace{-10pt}}

We consider algorithms that always satisfy \(\FDR(\rejset_\tau) \leq \delta\) for any number and configuration of nulls $\hypset_0$ and any choice of null and non-null distributions. In fact, our algorithm will produce a sequence of candidate rejection sets $(\rejset_t)$ that satisfies
\(
\sup_{\tau \in \Tcal} \FDR(\rejset_\tau) \leq \delta.
\)
This is a much stronger guarantee than the typical setting considered in the multiple testing literature. 
On the other hand, \(\TPR\) is a measurement of the power of the algorithm i.e.\ how many of the non-null hypotheses does the algorithm discover. Our implicit goal in the multiple testing problem is to maximize the number of true discoveries while not making too many mistakes i.e.\ keep the \(\FDR\) controlled.

In hypothesis testing, the set of null distributions $\Pcal_i$ for each arm $i$ is known, because the user defines the null hypothesis they are interested in testing. \textit{When the null hypothesis is false, the non-null distribution can be arbitrary}. Consequently, we can prove results about \(\FDR\), but we cannot prove guarantees about \(\TPR\) without several further assumptions on the non-null distributions, dependence across arms, etc. 
For a particular setting where we make such a set of assumptions, we demonstrate in \Cref{sec:SubGaussian} that we can prove \(\TPR\) guarantees for algorithms within our framework. Hence, our \(\FDR\) controlling framework is not vacuous as it includes powerful algorithms i.e.\ algorithms which make many true discoveries. However, our focus is primarily to show that the \(\FDR\) control of our framework is robust to a wide range of conditions.

Finally, note that in bandit multiple testing, one does not care about regret. The problem is more akin to \textit{pure exploration}, where we aim to find a \(\rejset\) with \(\FDR(\rejset_\stoptime) \leq \delta\) and large \(\TPR\) as quickly as possible.

Now that we have specified the problem we are interested in, we can introduce our main technical tools for ensuring \(\FDR\) control at stopping times: e-processes and p-processes.

\ifarxiv{}{\vspace{-5pt}}

\section{Technical preliminaries}
\ifarxiv{}{\vspace{-5pt}}

\subsection{E-processes versus p-processes}

An e-variable, \(E\), is a nonnegative random variable where \(\expect[E] \leq 1\) when the null hypothesis is true. In contrast, the more commonly used p-variable, \(P\), is defined to have support on \((0, 1)\) and satisfy \(\prob{P \leq \alpha} \leq \alpha \text{ for all }\alpha \in (0, 1)\) when the null hypothesis is true. To clearly delineate when we are discussing solely the properties of a random variable, we also use the terms ``e-value'' $e$ and ``p-value'' $p$ to refer to the realized values of a e-variable $E$ and a p-variable $P$ (their instantiations on a particular set of data). E-variables and p-variables are connected through Markov's inequality, which implies that \(1 / E\) is a p-variable (but $1/P$ is not in general an e-variable). 
Rejecting a null hypothesis is usually based on observing a small p-value or a large e-value.
For example, to control the false positive rate at 0.05 for a single hypothesis test, we reject the null when \(p \leq 0.05\) or when \(e \geq 20\). 

Since bandit algorithms operate over time, we define sequential versions of p-variables and e-variables. A p-process, denoted \((P_t)_{t \geq 1}\), is a sequence of random variables  such that \(\sup_{\tau \in \Tcal} \prob{P_\tau \leq \alpha} \leq \alpha\) for any \(\alpha \in (0, 1)\). In contrast, an e-process \((E_t)_{t \geq 1}\) must satisfy \(\sup_{\tau \in \Tcal} \expect[E_\tau] \leq 1\) (let \(E_\infty \coloneqq \limsup_{t \in \naturals} E_t\) and \(P_\infty \coloneqq \liminf_{t \in \naturals}P_t\)). These sequentially valid forms of p-variables and e-variables are crucial since we allow the bandit algorithm to stop and output a rejection set in a data-dependent manner. Thus, we must ensure the respective properties of p-variables and e-variables hold over all stopping times. 

These concepts are intimately tied to sequential testing and sequential estimation using confidence sequences~\citep{ramdas_admissible_2020}, but most importantly, nonnegative (super)martingales play a central role in the construction of efficient e-processes. To summarize, (a) for point nulls, all admissible e-processes are simply nonnegative martingales, and the safety property follows from the optional stopping theorem, (b) for composite nulls, admissible e-processes are either nonnegative martingales, or nonnegative supermartingales, or the infimum (over the distributions in the null) of nonnegative martingales. Associated connections to betting~\citep{waudby-smith_estimating_means_2021} are also important for the development of sample efficient algorithms and we discuss how we use betting ideas in \Cref{sec:Betting}. We also discuss some useful equivalence properties of p-processes in \Cref{sec:PProcesses}, while \Cref{subsec:Supermartingales} introduces supermartingales for the unfamiliar reader.

\ifarxiv{\paragraph}{\textbf}{Why use e-processes over p-processes?} \citet{wang_false_2020} describe a multitude of advantages outside of the bandit setting; these advantages also apply to the bandit setting but we do not redescribe them here for brevity.
However, we will describe multiple ways in which using e-variables instead of p-variables as a measure of evidence in the bandit setting allows for both better flexibility and sample complexity of the algorithm. While this question has been the focus of a recent line of work for hypothesis tests in general \citep{shafer_testing_betting_2021,vovk_evalues_calibration_2020,grunwald_safe_testing_2020,wang_false_2020}, we will explore how the properties of e-variables allow us to consider novel bandit setups and algorithms. In particular, e-variables allow us to be robust to arbitrary dependencies between statistics computed for each arm without additional correction. Further, we explore how e-processes can be merged under different conditions in \Cref{subsec:MultipleAgents} to facilitate incorporation of existing evidence and cooperation between multiple agents and present concrete ways to construct e-processes in \Cref{subsec:SampleComplexityProof,sec:Betting}.

Since any non-trivial bandit algorithm will base its sampling choice on the rewards attained so far for every arm, average rewards of each arm are biased and dependent on each other in complex ways even if the algorithm is stopped at a fixed time~\citep{nie2018adaptively,shin2019sample,shin2020conditional,shin2019bias}.
Even under a non-adaptive uniform sampling rule, an adaptive stopping rule can induce complex dependencies between reward statistics of each arm. When using both adaptive sampling and stopping, the dependence effects are only compounded. Nevertheless, e-variable based algorithms enable us to prove \(\FDR\) guarantees without assumptions on the sampling method. In contrast, procedures involving p-variables, such as the ones used in \JJ, require the test level of \(\alpha\) to be corrected by a factor of at least \(\log(1 / \alpha)\) when rewards are independent across arms, and a factor of \(\log k\) otherwise. We expand on this in \Cref{sec:BH}.

\subsection{Multiple testing procedures with \(\FDR\) control}
\label{sec:BH}
We now introduce two multiple testing procedures that output a rejection set with provable \(\FDR\) control. We will first describe the guarantees provided by the BH procedure \citep{benjamini_controlling_1995-2}, a classic multiple testing procedure that operates on p-variables. Then, we will describe e-BH, the e-variable analog of BH. Our key message in this section is that classical BH will have looser or tighter control of the \(\FDR\) based upon the dependence structure of the p-variables it is operating on. On the other hand, e-BH provides a consistent guarantee on the \(\FDR\) even when the e-variables are arbitrarily dependent. 
Both procedures take an input parameter \(\alpha \in (0, 1)\) that controls the degree of \(\FDR\) guarantee (i.e. test level).

\ifarxiv{\paragraph}{\textbf}{Benjamini-Hochberg (BH) requires corrections for dependence and self-consistency.} 
A set \(\rejset\) of p-values is called \textit{p-self-consistent} \citep{blanchard_two_simple_2008} at level \(\alpha\) iff:
\begin{align}
    \max_{i \in \rejset}\ p_i \leq \ifarxiv{\frac}{\tfrac}{|\rejset|\alpha}{k}.
    \label{eqn:PComplianceDef}
\vspace{-5pt}
\end{align} 
The BH procedure with input \(p_1, \dots, p_k\) outputs the largest p-self-consistent set w.r.t.\ the input, which we denote \(\BH[\alpha](p_1, \dots, p_k)\). We must also define a condition on the joint distribution of \(P_1, \dots, P_k\), which is called positive regression dependence on subset (PRDS). A formal definition is provided in \citet{benjamini_control_false_2001}, and it is sufficient for our purposes to think of this condition as positive dependence between \(P_1, \dots, P_k\), with independence being a special case. Now, we describe the \(\FDR\) control of the BH procedure.
\begin{fact}[BH FDR control. \citet{benjamini_controlling_1995-2,benjamini_control_false_2001}]
Let \(\rejset = \BH[\alpha](p_1, \dots, p_k)\). If \(P_1, \dots P_k\) are PRDS, then \(\FDR(\rejset) \leq \alpha\). Otherwise, under arbitrary dependence amongst \(P_1, \dots P_k\), the BH procedure ensures \(\FDR(\rejset) \leq  \alpha  \ell_k\), where \(\ell_k \equiv \sum_{i = 1}^k 1/k \approx \log k\).
\label{fact:BHFDR}
\end{fact} 
\vspace{-5pt}
Thus, in the case of arbitrary dependence, the \(\FDR\) control of BH is larger by a factor of \(\ell_k \approx \log k\). A larger \(\FDR\) guarantee is provided for arbitrary p-self-consistent sets.

\begin{fact}[P-self-consistent \(\FDR\) control. \citet{su_fdr-linking_2018,blanchard_two_simple_2008,wang_false_2020}]
If \(\rejset\) is p-self-consistent at level $\alpha$ and \(P_1, \dots, P_k\) satisfy PRDS,
\footnote{\citet{su_fdr-linking_2018} technically employs a \emph{slightly} weaker condition which implies PRDS, and refers to self-consistency as ``compliance'' (or, better said, compliance is a special case of self-consistency).} 
then \(\FDR(\rejset) \leq \alpha(1 + \log(1 / \alpha))\). Otherwise, when there is arbitrary dependence among \(P_1, \dots, P_k\), \(\FDR(\rejset) \leq  \alpha\ell_k\) (consequence of Propositions 2.7 and 3.7 from \citet{blanchard_two_simple_2008} and Proposition 5.2 from \citet{wang_false_2020}).
\label{fact:NewPComplianceFDR}
\end{fact}

\vspace{-5pt}

These two facts do not imply each other; the BH procedure outputs the largest self-consistent set and has a stronger or equivalent error guarantee under either type of dependence. While it may seem like we should always use BH and the guarantee from \Cref{fact:BHFDR} to form a rejection set, we elaborate in \Cref{sec:Dependence} on how we can use \Cref{fact:NewPComplianceFDR} to provide \(\FDR\) control for BH when the p-variables are not necessarily PRDS, and in settings where we may not directly use BH. 

\ifarxiv{\paragraph}{\textbf}{e-BH needs no correction for dependence or self-consistency.} The e-BH procedure created by \citet{wang_false_2020} uses e-variables instead of p-variables and proceeds similarly to the BH procedure. In this case, let \(e_1, \dots, e_k\) be the realized e-values for a set of e-variables \(E_1, \dots, E_k\). Define \(e_{[i]}\) to be the \(i\)th largest e-value for \(i \in [k]\). A set \(\rejset\) is \textit{e-self-consistent} at level \(\alpha\) iff \(\rejset\) satisfies the following:
\begin{align}
    \min_{i \in \rejset}\ e_i \geq \ifarxiv{\frac}{\tfrac}{k}{\alpha|\rejset|}.
    \label{eqn:EComplianceDef}
\vspace{-5pt}
\end{align} 
The e-BH procedure outputs the largest e-self-consistent set, which we denote by \(\EBH[\alpha](e_1, \dots, e_k)\). For e-variables, the same guarantee applies for all e-self-consistent sets and under all dependence structures.
\begin{fact}[E-variable self-consistency FDR control. \citet{wang_false_2020}]
If \(\rejset\) is e-self-consistent at level $\alpha$, then \(\FDR(\rejset) \leq \alpha\) regardless of the dependence structure.
\label{fact:EComplianceFDR}
\end{fact} 
\vspace{-5pt}
\begin{revision}
All \(\FDR\) bounds discussed in \Cref{fact:BHFDR,fact:NewPComplianceFDR,fact:EComplianceFDR} are optimal, in the sense that there exist e-variable/p-variable distributions with an \(\FDR\) that is arbitrarily close or equivalent to the stated bound. Consequently, e-variables are more advantageous, since their \(\FDR\) control does not change under different types of dependence as opposed to the factor of $1 + \log(1/\alpha)$ or $\log k$ p-variables pay on the \(\FDR\) for different settings.
\end{revision}

In the case where p-variables can only be constructed as \(P = 1 / E\), where \(E\) is an e-variable, the rejection sets output by BH and e-BH are identical. However, the e-self-consistency guarantee in \Cref{fact:EComplianceFDR} provides identical or tighter \(\FDR\) control than the BH procedure guarantee in \Cref{fact:BHFDR} or p-self-consistency guarantee in \Cref{fact:NewPComplianceFDR}. Thus, e-variables and e-BH offer a degree of robustness against arbitrary dependence, since any algorithm using e-BH does not have to adjust \(\alpha\) to guarantee the same level of \(\FDR(\rejset)\leq \delta\) for a fixed \(\delta\) under different dependence structures. We now provide a meta-algorithm that utilizes p-self-consistency and e-self-consistency to guarantee \(\FDR\) control in the bandit setting.

\ifarxiv{}{\vspace{-5pt}}
\section{Decoupling exploration and evidence: a unified framework}
\label{sec:Framework}
\ifarxiv{}{\vspace{-5pt}}

We propose a framework for bandit algorithms that separates each algorithm into an \textbf{exploration} component and an \textbf{evidence} component; 
\Cref{alg:Framework} specifies a meta-algorithm combining the two.
\vspace{-10pt}

\begin{algorithm}
\label{alg:Framework}
\caption{A meta-algorithm for bandit multiple testing that decouples exploration and evidence. 
The evidence component can track p-processes or e-processes for each arm and use BH or e-BH.}
\KwIn{Exploration component \((\EC_t)\), stopping rule \(\stoptime\),
Let \((p_{1, t}), \dots, (p_{k, t})\) and \((e_{1, t}), \dots, (e_{k, t})\) denote the realized values of p-processes and e-processes, respectively. Let the desired level of \(\FDR\) control be \(\delta \in (0, 1)\). Let \(\delta'\) be the correction of \(\delta\) for BH based upon the dependencies of \(X_{1, t}, \dots, X_{k, t}\). Set $D_0 = \emptyset$.}
\For{\(t\) in \(1 \dots\)}{
    \(\sampleset_t \coloneqq \EC_t(D_{t-1}) \subseteq [k]\)\\
    Obtain rewards for each $i \in \sampleset_t$, and update data $D_t:=D_{t-1} \cup \{(i, t, X_{i, t}): i \in \sampleset_t\}$.\\
    Update e-process or p-process for each queried arm (summarizing evidence against each null).\\
    \(\rejset_t \coloneqq 
    \begin{cases}
    \BH[\delta'](p_{1, t}, \dots, p_{k, t}) \text{ or arbitrary p-self-consistent set} & \text{if using p-variables}\\
    \EBH[\delta](e_{1, t}, \dots, e_{k, t}) \text{ or arbitrary e-self-consistent set}& \text{if using e-variables}
    \end{cases}\)\\
    \lIf{\(\stoptime = t\)}{stop and \Return \(\rejset_t\)}
}
\end{algorithm}
\vspace{-10pt}

\ifarxiv{\paragraph}{\textbf}{Exploration component.} This is a sequence of functions \((\EC_t)\), where \(\EC_t: \filtration_{t - 1} \mapsto \Kcal\) specifies the queried arms $\sampleset_t := \EC_t(D_{t-1})$, and $D_t:=\{(i,j, X_{i, j}): j \leq t, i \in \sampleset_j\}$ is the observed data. $\EC_t$ is ``non-adaptive'' if it does not depend on the data, but only on some external randomness $\privrv$.
Regardless of how the exploration component \((\EC_t)\) is constructed, our framework guarantees that \(\FDR(\rejset) \leq \delta\) for a fixed \(\delta\). Similarly, \(\stoptime\) is adaptive if it depends on the data, and is not determined purely by \(U\).

\ifarxiv{\paragraph}{\textbf}{Evidence component.} The \(\FDR\) control provided by \Cref{alg:Framework} is solely due to the formulation of the candidate rejection set, \(\rejset_t \subseteq [k]\),  at each time \(t \in \naturals\) in the evidence component. This construction is completely separate from \((\EC_t)\).
Critically, \((\rejset_t)\) satisfies \(\FDR(\rejset_{\tau}) \leq \delta\) for any stopping time \(\tau \in \Tcal\). This is accomplished by applying BH or e-BH to p-processes or e-processes, respectively. At stopping time \(\tau\), \(P_{i, \tau}\) is a p-variable when \((P_{i, t})\) is a p-process, and similarly \(E_{i, \tau}\) is an e-variable when \((E_{i, t})\) is an e-process. Thus, \(\rejset_\tau\) is the result of applying BH to p-variables or e-BH to e-variables.

Consequently, the aforementioned framework allows us to guarantee \(\sup_{\tau \in \Tcal} \FDR(\rejset_\tau) \leq \delta\) in a way that is agnostic to the exploration component. For completeness, we do discuss some generic exploration strategies in \Cref{sec:GenericAlgs}. In the next section, we will formalize these guarantees and discuss the benefits afforded by using e-variables and e-BH in this framework instead of p-variables and BH.

\subsection{\(\FDR\) control under different dependence structures}
\label{sec:Dependence}

In the general combinatorial bandit setting, different dependence structures affect the choice of \(\delta'\) that ensures \(\FDR\) control at \(\delta\) in the p-variable and BH case. \Cref{table:Dependence} summarizes the guarantees and choices of \(\delta'\) for each type of dependence. Prior work on hypothesis testing in the bandit setting by \JJ\ has only considered the non-combinatorial bandit case where \(X_{1, t}, \dots, X_{k, t}\) are independent. Critically, \JJ\ employ BH and p-variables in their algorithm, and the \(\FDR\) guarantee of BH changes based on the dependencies between reward distributions. On the other hand, choosing \(\alpha = \delta\) for e-BH is sufficient to guarantee \(\FDR\) control at level \(\delta\) for any type of dependence between e-variables, but only sufficient for BH in the non-adaptive, PRDS \(X_{1, t}, \dots, X_{k, t}\) setting. We show that there is a wide range of dependence structures that require different degrees of correction for BH.  Specifically, we will set an appropriate choice of \(\delta'\) in each of these situations such that \Cref{alg:Framework} with p-variables can ensure \(\FDR\) control level \(\delta\). We include proofs of all results in this section in \Cref{sec:DependenceProofs}.

\setlength{\tabcolsep}{6pt} 
\renewcommand{\arraystretch}{1.1} 
\begin{table}[h!]
    \ifarxiv{}{\vspace{-10pt}}

    \caption{\(\FDR\) control for BH, and the \(\delta'\) to ensure \(\delta\) control of \(\FDR\) in \Cref{alg:Framework} under different dependence structures and adaptivity of \((\EC_t)\). Adaptivity and arbitrary dependence both require extra correction for BH, but \textit{any e-self-consistent procedure provides \(\FDR(\rejset) \leq \alpha\) in all settings in the table.}}
    \centering
	\begin{tabular}{c|l|l|}
\cline{2-3}
\multicolumn{1}{l|}{}                                        & \multicolumn{2}{c|}{\textbf{Dependence of} \(X_{1, t}, \dots, X_{k, t}\)}                                                             \\ \hline
\multicolumn{1}{|c|}{\textbf{Adaptivity} of \((\EC_t)\) and \(\stoptime\)}              & \multicolumn{1}{c|}{\textit{independent}}                        & \multicolumn{1}{c|}{\textit{arbitrarily dependent}} \\ \hline
\multicolumn{1}{|c|}{\multirow{2}{*}{\textit{non-adaptive}}}     & $\FDR(\rejset) \leq \alpha$  &                                                     \\
\multicolumn{1}{|c|}{}                                       & \(\delta' = \delta\) & $\FDR(\rejset) \leq \alpha \log k$                  \\ \cline{1-2}
\multicolumn{1}{|c|}{\multirow{2}{*}{\textit{adaptive}}} & $\FDR(\rejset) \leq \alpha((1 + \log(1 / \alpha)) \wedge \log k)$                                      & \(\delta' = \delta / \log k\) (Prop.~\ref{prop:AdaptiveDepPFDR})                      \\
\multicolumn{1}{|c|}{}                                       & \(\delta' = c_{\delta} \vee (\delta / \log k)\) (Prop.~\ref{prop:AdaptiveIndPFDR})                                            &                                                     \\ \hline\hline
\multicolumn{3}{|c|}{Any \textbf{e-self-consistent procedure} ensures \(\FDR(\rejset)\leq \alpha\) in all settings and sets \(\alpha = \delta\).}\\
\hline
\end{tabular}
\label{table:Dependence}
\end{table}


\ifarxiv{\paragraph}{\textbf}{Adaptive \((\EC_t)\) and independent \(X_{1, t}, \dots, X_{k, t}\).} \JJ\ consider this case in the non-combinatorial bandit setting, but their insights and techniques also can be extended to the combinatorial setting. We give a sketch of their proof here, and produce the full proof in \Cref{sec:DependenceProofs}. In the language of self-consistency (not explicitly used in \JJ), \JJ\ make the key insight that \textit{running BH on the p-variables for each arm produces a rejection set that is actually p-self-consistent with a different set of independent p-variables.} Define \(P_1^*, \dots, P_k^*\), where \(P_i^* = \inf_{t \in \naturals} P_{i, t}\) for each \(i \in [k]\) i.e.\ each arm's p-variable in the infinite sample limit. Since \((P_{i, t})\) is a p-process for each arm \(i \in [k]\), the corresponding \(P_i^*\) is a p-variable (\Cref{prop:PEquiv} in \Cref{sec:PProcesses}). Further, \(P_1^*, \dots, P_k^*\) are independent because \(X_{1, t}, \dots, X_{k, t}\) are independent. 
By definition of \(P_1^*, \dots, P_k^*\), \(p_i^* \leq p_{i, t} \) for any \(i \in [k]\) and any \(t \in \naturals\). 
Thus, \(\rejset_{\stoptime}\) is  p-self-consistent w.r.t.\ \(p_{1}^*, \dots, p_k^*\), and has its \(\FDR\) bounded by \(\alpha(1 + \log(1 / \alpha))\) due to \Cref{fact:NewPComplianceFDR}. At the same time, the arbitrary dependence guarantee from \Cref{fact:BHFDR} still applies. Combining these facts, we achieve the following guarantee:
\begin{proposition}
When \((\EC_t)\) is adaptive and \(X_{1, t}, \dots, X_{k, t}\) are independent, \Cref{alg:Framework} with p-processes and an arbitrary p-self-consistent set guarantees \(\sup_{\tau \in \Tcal} \FDR(\rejset_\tau) \leq \delta\)  if \(\delta' \leq c_{\delta} \vee \delta/\ell_k\), where for any $\delta \in (0,1)$, define $c_\delta \leq \delta$ as the solution to
\(
c_{\delta}(1 + \log(1 / c_{\delta})) = \delta.
\)
\label{prop:AdaptiveIndPFDR}
\end{proposition}
\vspace{-5pt}
Note that \Cref{prop:AdaptiveIndPFDR} is valid for any p-self-consistent set since p-self-consistency is the only property required of the output set to prove the result. \JJ\ prove a similar bound to \Cref{prop:AdaptiveIndPFDR}. However, they used a larger \(\FDR\) bound for p-self-consistent sets with worse constants (which was subsequently improved by~\citet{su_fdr-linking_2018} as presented earlier), and they only considered the non-combinatorial case. \Cref{prop:AdaptiveIndPFDR} uses an optimal bound on p-self-consistent sets from \Cref{fact:NewPComplianceFDR}, and is valid in our combinatorial bandit setup.

\ifarxiv{\paragraph}{\textbf}{Adaptive \((\EC_t)\) and arbitrarily dependent \(X_{1, t}, \dots, X_{k, t}\).}
In the general combinatorial bandit setting, where the algorithm chooses a subset of arms or ``superarm'' at each time to jointly sample from, we will have multiple samples from multiple arms in the same time step, and \(X_{1, t}, \dots, X_{k, t}\) can be arbitrarily dependent. Consequently, the p-variables corresponding to each arm can also be arbitrarily dependent. For example, a superarm could consist of all arms, and the sampling rule could be to just sample this superarm that encompasses all arms. Then, the p-variable distribution would directly depend on the reward distribution of the arms. Thus, we can provide the following guarantee by \Cref{fact:BHFDR} when using p-variables as a result of \Cref{fact:EComplianceFDR}.

\begin{proposition}
When \((\EC_t)\) is adaptive and \(X_{1, t}, \dots, X_{k, t}\) are dependent, \Cref{alg:Framework} with p-variables and BH guarantees \(\sup_{\tau \in \Tcal} \FDR(\rejset_\tau) \leq \delta\) if \(\delta' \leq \delta / \ell_k\).
\label{prop:AdaptiveDepPFDR}
\end{proposition}
\ifarxiv{}{\vspace{-5pt}}

Finally, consider a setting structured setting where we cannot output the rejection set of BH. Such a constraint often occurs in directed acyclic graph (DAG) settings where there is a hierarchy among hypotheses that restricts which rejection sets are allowed \citep{ramdas_sequential_algorithm_2019,lei_general_interactive_2020}. Instead, we would like to output the largest self-consistent set that respects the structural constraints. By \Cref{fact:NewPComplianceFDR}, we get the following \(\FDR\) control.
\begin{proposition}
If \((\EC_t)\) is adaptive and \(X_{1, t}, \dots, X_{k, t}\) are dependent, \Cref{alg:Framework} with p-variables that outputs an arbitrary p-self-consistent \(\rejset_t\) guarantees \(\sup_{\tau \in \Tcal} \FDR(\rejset_\tau) \leq \delta\) if \(\delta' \leq c_{\delta} / \ell_k\).
\label{prop:AdaptiveStructurePFDR}
\end{proposition}
We explore the structured setting with greater depth in \Cref{subsec:Structured}. Unlike p-variables, e-variables do not need correction in any of the aforementioned settings.
\begin{proposition}
When \((\EC_t)\) is adaptive and \(X_{1, t}, \dots, X_{k, t}\) are dependent, \Cref{alg:Framework} with e-variables, which runs e-BH at level \(\delta\) or outputs a e-self-consistent set at level \(\delta\), guarantees \(\sup_{\tau \in \Tcal}\FDR(\rejset_\tau) \leq \delta\).
\label{prop:AdaptiveDepEFDR}
\end{proposition} 
\vspace{-10pt}
Thus, running e-BH (or any e-self-consistent procedure) at level \(\delta\) is valid for any choice of \((\EC_t)\) and type of dependence. Now, we give an example where \(X_{1, t}, \dots, X_{k, t}\) might be arbitrarily dependent.

\subsection{Illustrative examples to demonstrate flexibility of the framework}

Below, we briefly describe a set of nontrivial illustrative examples to showcase the flexibility of our framework. In most of the cases below, a p-process approach would have to correct for dependence and/or self-consistency in different case-specific ways, rendering it more conservative and requiring careful arguments to justify \(\FDR\) control. However, working with our unified framework is easy, handling both self-consistency and dependence issues in the same breath and without any changes to the algorithm or analysis. The data scientist can focus on designing powerful e-processes \emph{for each arm separately} and let the modular framework correct for the multiplicity aspect.

\ifarxiv{}{\setlength{\intextsep}{0pt}}%
\ifarxiv{}{\setlength{\columnsep}{6pt}}%

\ifarxiv{\paragraph}{\textbf}{Example: sampling nodes on a graph.} A scenario where \(X_{1, t}, \dots, X_{k, t}\) may naturally have dependence is when each arm corresponds to a node on a graph. The superarms in this situation could be defined w.r.t.\ to a graph constraint e.g.\ ``two nodes connected by an edge'' or ``a node and its neighbors''. Graph bandits has been studied in the regret setting \citep{mannor_bandits_2011-1} and have many real world applications \citep{valko_bandits_graphs_2016}. We could imagine a scenario where low power sensors in a sensor network can only communicate locally. A centralized algorithm is tasked with querying the sensors to find those with high activity. A sensor may only provide activity information about itself and nearby sensors, and this data can be arbitrarily dependent across the sensors. 
\Cref{fig:GraphBandit} illustrates a superarm in this situation. 

\ifarxiv{\begin{figure}[h]
\centering
\includegraphics[width=0.3\textwidth]{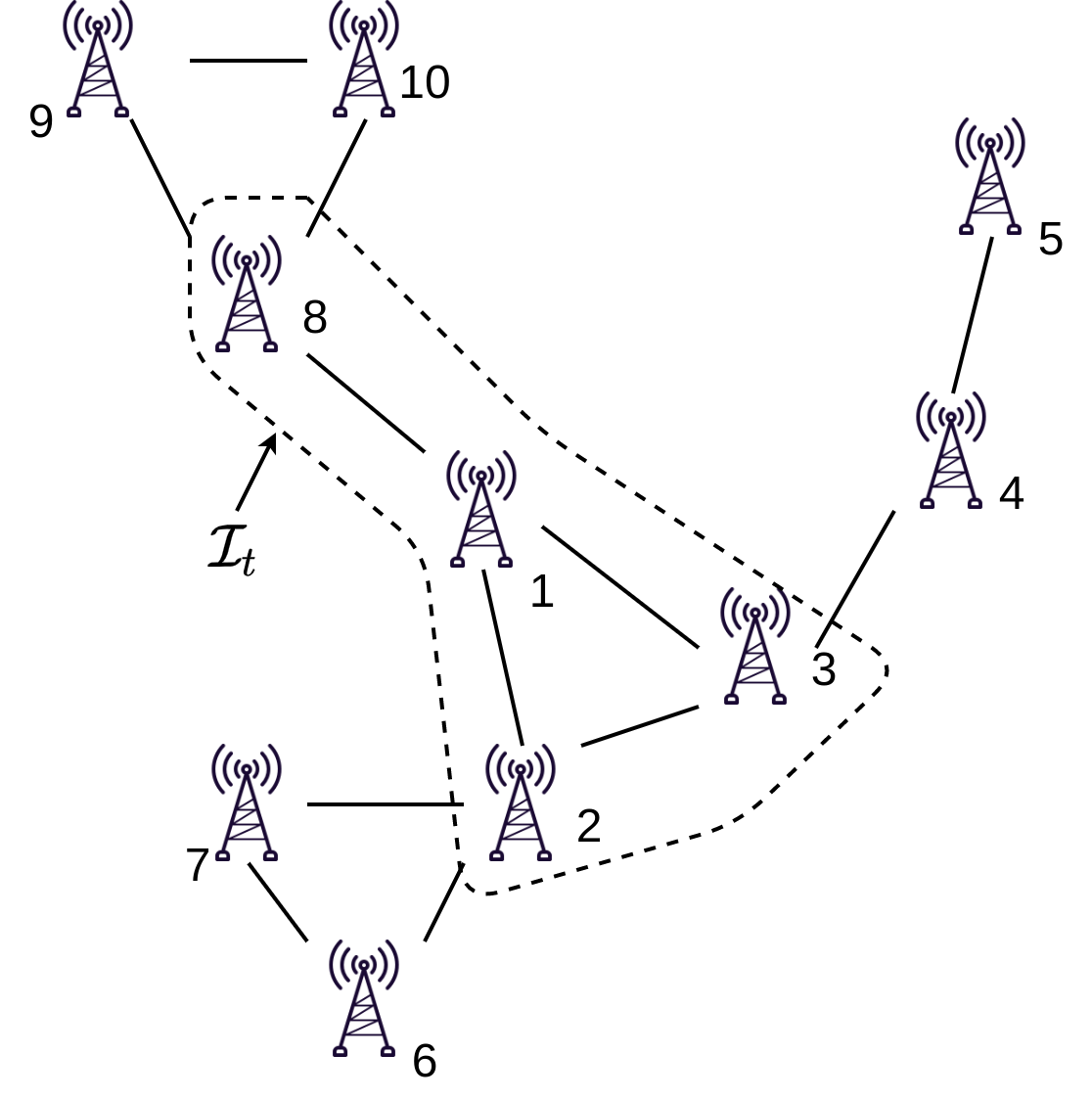}
\caption{A superarm consists of a node and all its neighbors. The dotted line captures \(\sampleset_t\), the superarm around node 1.}
\label{fig:GraphBandit}
\end{figure}}{
\begin{wrapfigure}[15]{r}{0.3\textwidth}
\centering
\includegraphics[width=0.3\textwidth]{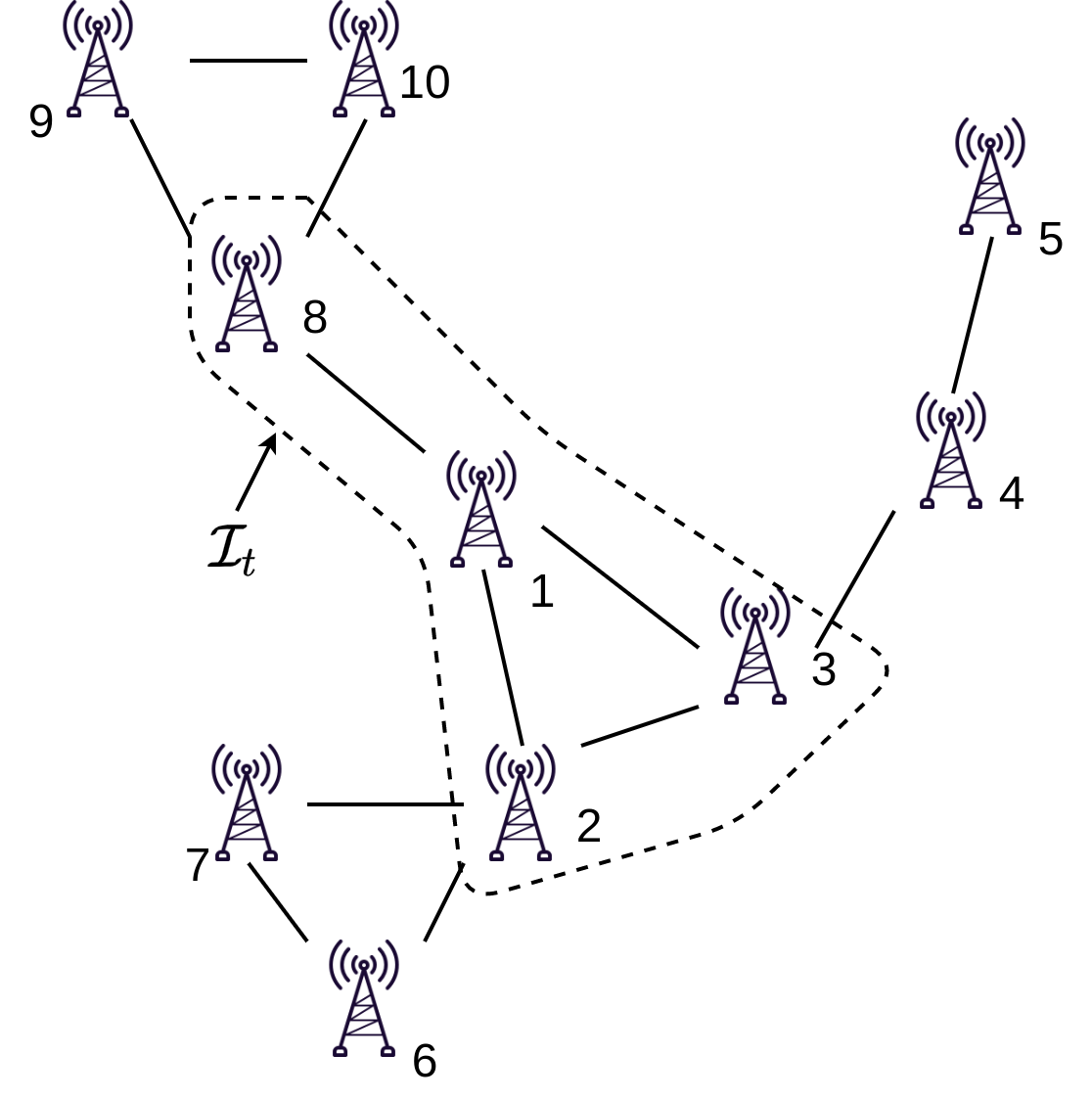}
\caption{A superarm consists of a node and all its neighbors. The dotted line captures \(\sampleset_t\), the superarm around node 1.}
\label{fig:GraphBandit}
\end{wrapfigure}
}
\ifarxiv{}{\vspace{-5pt}}

In such a setting, if \Cref{prop:AdaptiveDepPFDR} is used to guarantee \(\FDR(\rejset) \leq \delta\) with p-variables, it pays a \(\log k\) correction, while \Cref{prop:AdaptiveDepEFDR} can guarantee e-variables need no correction. We simulate this setting in \Cref{subsec:GraphSimulations}, and show these differences empirically.
We also discuss some other examples in the appendix that we will summarize here. 

\begin{itemize}[leftmargin=9pt]
    \item \textbf{Multiple agents} (\Cref{subsec:MultipleAgents}): Consider the setting where multiple agents are operating on the same bandit, and we want to aggregate the evidence for rejection across agents. For e-processes, we present an algorithm for merging e-values that maintains \(\FDR\) control.
    \item \textbf{Structured rejection sets} (\Cref{subsec:Structured}): We illustrate the difference between self-consistency guarantees for p-variables and e-variables when a DAG hierarchy is imposed upon the hypotheses.
    \item \textbf{Multi-arm hypotheses} (\Cref{subsec:MultipleArms}) A hypothesis may concern the reward distributions of multiple arms e.g.\ are the means of two different arms equivalent? We provide \(\FDR\) guarantees even when hypotheses and arms are not matched one-to-one.
    \item \textbf{Streaming data setting} (\Cref{subsec:StreamingData}) Our methods also naturally extend to the streaming setting when the algorithm views the rewards of every at each time step.
\end{itemize}

Now that we have shown \(\FDR\) is controlled using e-variables in a way that is robust to the underlying dependence structure, we analyze the sample complexity of achieving a high \(\TPR\) using e-variables when the rewards are independent and sub-Gaussian.







\ifarxiv{}{\vspace{-5pt}}
\section{E-process sample complexity guarantees for sub-Gaussian arms}
\label{sec:SubGaussian}
\ifarxiv{}{\vspace{-5pt}}

We provide sample complexity guarantees for the sub-Gaussian setting that has been the focus of existing methodology by \JJ\ in bandit multiple testing. We explicitly define e-processes and an exploration component \((\EC_t)\) that will have sample complexity bounds matching those of the algorithm in \JJ, which uses p-variables. Specifically, we will consider the standard bandit setting where \(|\sampleset_t| = 1\) and \(\nu_i\) is \(1\)-sub-Gaussian for each \(i \in [k]\). Denote the means of each arm \(i \in [k]\) as \(\mu_i = \expect[X_{i, t}]\) for all \(t \in \naturals\). The goal is to find many arms where \(\mu_i > \mu_0\), where we set \(\mu_0 = 0\) to be the mean of a reward distribution under the null hypothesis. Thus, we define \(\hypset_0 = \{i \in [k]: \mu_i \leq \mu_0\}\) and \(\hypset_1 = \{i \in [k]:\mu_i > \mu_0\}\).  Our framework ensures that \(\FDR(\rejset) \leq \delta\), and we also want to achieve \(\TPR(\rejset) \geq 1 - \delta\) with small sample complexity. Proofs of the results from this section are in \Cref{subsec:SampleComplexityProof,subsec:DMEProcessProof}. As an aside, we also discuss what hypotheses we can test when the reward distribution is not necessarily independent across \(t \in \naturals\), but the conditional distribution of the rewards still satisfy certain sub-Gaussian guarantees in \Cref{sec:ACE}.

Our e-process of choice is the \textbf{discrete mixture e-process} from \citet{howard2021time}:

\hnonarxivwrapfigure{6}{r}{0.64\textwidth}{
\begin{subequations}
\begin{gather}
    E_{i, t}^{\DM}(\mu_0) \coloneqq \sum\limits_{\ell = 0}^\infty w_\ell \exp\left( \sum\limits_{j = 1}^{T_i(t)}\lambda_{\ell}(X_{i, t_i(j)} - \mu_0) - \lambda_{\ell}^2 / 2\right), \numberthis \label{eqn:DiscreteMixtureE}\\
     \text{ where } \lambda_\ell \coloneqq \tfrac{1}{e^{\ell + 5/2}}\ \text{ and } w_\ell \coloneqq  \tfrac{2(e - 1)}{e(\ell + 2)^2}\ \text{for }\ell \in \naturals_0.
     \label{eqn:DMConstants}
\end{gather}
\end{subequations}
}

\begin{proposition}
\(E_{i, t}^{\DM}\) is an e-process when \(\nu_i\) is \(1\)-sub-Gaussian and \(i \in \hypset_0\).
\label{prop:DMEProcess}
\end{proposition} 

Denote \(\gap_i \equiv \mu_i - \mu_0\) for \(i \in \hypset_1\) and \(\gap \equiv \min_{i \in \hypset_1} \gap_i\). When \(i \in \hypset_0\), let \(\gap_i \equiv \min_{j \in \hypset_1} \mu_i - \mu_0 = \gap + (\mu_i - \mu_0)\). First, we recall a time-uniform bound on the sample mean \(\widehat{\mu}_t\).
\begin{fact}[JJ, \citet{kaufmann_complexity_bestarm_2016,howard2021time}]
Let \(X_1, X_2, \dots\) be i.i.d.\  draws from a \(1\)-sub-Gaussian distribution with mean $\mu$. Consider the boundaries defined in \eqref{eqn:Boundaries}. Let \(\varphi\) be one of these boundaries. Then, \nonarxivalign{\prob{\exists t \in \naturals: |\widehat{\mu}_t - \mu| > \varphi(t, \delta)} \leq \delta}
for any \(\delta \in (0, 1)\) if \(\varphi \in \{\varphi^0, \isphi\}\) and any \(\delta \in (0, 0.1]\) if \(\varphi = \varphi^{\mathrm{JJ}}\).
\label{fact:LILBound}
\end{fact}


\hnonarxivwrapfigure{8}{r}{0.67\textwidth}{
\begin{subequations}
\begin{align}
\varphi^{0}(t, \delta) &\coloneqq \sqrt{\tfrac{4 \log(\log_2(2t) / \delta)}{t}},\\
\varphi^{\mathrm{JJ}}(t, \delta) &\coloneqq \sqrt{\tfrac{2\log(1 / \delta) + 6 \log \log(1 / \delta) + 3\log(\log(et / 2))}{t}},\\
\isphi(t, \delta) &\coloneqq \sqrt{\tfrac{2.89\log\log(2.041t) + 2.065 \log \left(\frac{4.983}{\delta}\right)}{t}}.
\end{align}
\label{eqn:Boundaries}
\end{subequations}
}

We will use \(\varphi\) to refer to an arbitrary boundary from \Cref{fact:LILBound}. All of the $\varphi$ are time-uniform boundaries that yield confidence sequences for the mean. Note that \(\varphi^0\) is generally larger than the other boundaries, so we use \(\varphi^0\) as the default boundary in our proofs, and we explore how different choices of \(\varphi\) affect empirical performance in \Cref{subsec:PVariableSimulations}. Now, we can define the algorithm from \JJ\ in \eqref{eqn:PAlgorithm}, which consists of an exploration policy based on an upper confidence bound (UCB) of the mean reward (specified by a singleton set \(\sampleset_t = \{I_t\}\)) and a p-variable derived from \Cref{fact:LILBound}.  

In \eqref{eqn:UCB}, we denote the sample mean at time \(t\) of each arm \(i \in [k]\) by \(\widehat{\mu}_{i, t}\). Let \(f \lesssim g\) denote \(f\) asymptotically dominates \(g\) i.e.\ there exist \(c > 0\) that is independent of the problem parameters such that \(f \leq cg\). \JJ\ prove the following sample complexity guarantee for their algorithm.

\hnonarxivwrapfigure{4}{r}{0.54\textwidth}{
\begin{subequations}
\begin{align}
    \ifarxiv{}{\small}
    I_t &= \argmax_{i \in [k]\setminus \rejset_{t - 1}} \widehat{\mu}_{i, t - 1} + \varphi(T_i(t - 1), \delta), \label{eqn:UCB}\\
    P_{i, t} &\equiv \inf\{\rho \in [0, 1]: |\widehat{\mu}_{i, t} - \mu_0| > \varphi(t, \rho)\}.
    \label{eqn:PVariable}
\end{align} 
\label{eqn:PAlgorithm}
\end{subequations}
}
\begin{fact}[From JJ]
Let \((\EC_t)\) output \(\sampleset_t = \{I_t\}\), and let \(I_t\) and \(P_{i, t}\) be specified by Alg.~\ref{eqn:PAlgorithm} with \(\varphi = \varphi^0\). Then, \Cref{alg:Framework} will always guarantee \(\sup_{\tau \in \Tcal}\FDR(\rejset_\tau) \leq \delta\). With at least \(1 - \delta\) probability, there will exist \nonarxivalign{T \lesssim \left(\sum_{i = 1}^k \gap_i^{-2}\log\log \gap_i^{-2} + \gap_i^{-2}\log(k / \delta)\right) \wedge k\gap^{-2}\log(\log(\gap^{-2}) / \delta)} such that 
\(\TPR(\rejset_t) \geq 1 - \delta\) for all \(t \geq T\).
\label{fact:PAlgorithmSampleComplexity}
\end{fact}

We show that we can match the sample complexity bounds of \Cref{fact:PAlgorithmSampleComplexity} with e-variables.
\begin{theorem}
    Let \(\nu_i\) be \(1\)-sub-Gaussian for \(i \in [k]\).
    Set \((\EC_t)\) so \(\EC_t\) outputs \(\{I_t\}\) from \eqref{eqn:UCB} for all \(t \in \naturals\) and \(E_{i, t}\) to \(E_{i, t}^{\DM}\). 
   \Cref{alg:Framework} ensures \(\sup_{\tau \in \Tcal}\FDR(\rejset_{\tau}) \leq \delta\) and, with at least \(1 - \delta\) probability, there exists \nonarxivalign{T \lesssim \left(\sum_{i = 1}^k \gap_i^{-2}\log\log \gap_i^{-2} + \gap_i^{-2}\log(k / \delta)\right) \wedge k\gap^{-2}\log(\log(\gap^{-2}) / \delta)} such that \(\TPR(\rejset_t) \geq 1 - \delta\) for all \(t \geq T\).
    \label{thm:DiscreteESampleComplexity}
\end{theorem}

In addition to matching theoretical guarantees, we show in the following section that e-variables and e-BH perform empirically as well or better than p-variables and BH through numerical simulations.

\ifarxiv{}{\vspace{-5pt}}
\section{Numerical simulations}
\label{sec:Experiments}
\ifarxiv{}{\vspace{-5pt}}

We perform simulations for the sub-Gaussian setting discussed in \Cref{sec:SubGaussian} to demonstrate that our version of \Cref{alg:Framework} using e-variables is empirically as efficient as the algorithm of \JJ, which uses p-variables (code available \href{https://github.com/neilzxu/e_bmt}{here})
. However, unlike \JJ, our algorithm does not use a corrected level \(\delta'\) based upon the dependence assumptions among \(X_{1, t}, \dots, X_{k, t}\) to guarantee \(\FDR\) is controlled at level \(\delta\). We explore additional simulations of combinatorial semi-bandit settings with dependent \(X_{1, t}, \dots, X_{k, t}\) in \Cref{sec:AdditionalSimulations} that show the benefit of using e-variables over p-variables in our framework.

\ifarxiv{\paragraph}{\textbf}{Simulation setup} Let \(\nu_i = \Gaussian(\mu_i, 1)\) where \(\mu_i = \mu_0 = 0\) if \(i \in \hypset_0\) and \(\mu_i = 1/2\) if \( i \in \hypset_1\). We consider 3 setups, where we set the number of non-null hypotheses to be \(|\hypset_1| = 2, \log k\), and \(\sqrt{k}\), to see the effect of different magnitudes of non-null hypotheses on the sample complexity of each method. We set \(\delta = 0.05\) and compare 4 different methods. We compare the same two different exploration components for both e-variables and p-variables. The first exploration component we consider is simply uniform sampling across each arm (Uni). The second is the UCB sampling strategy described in \eqref{eqn:UCB}. When using BH, our formulation for p-variables is \eqref{eqn:PVariable}, which is the same as \JJ. Like \JJ, we set \(\varphi = \varphi^{\mathrm{\JJ}}\) in our simulations. When using e-BH, we set our e-variables to \(\pmh_{i, t} \coloneqq \prod_{j = 1}^{T_i(t)}\exp(\lambda_{i, t_i(j)}(X_{i, t_i(j)} - \mu_0) - \lambda_{i, t_i(j)}^2 / 2)\) with \(\lambda_{i, t} = \sqrt{\frac{2 \log(2 / \alpha) }{T_i(t) \log (T_i(t) + 1)}}\), which is the default choice of \(\lambda_{i, t}\) suggested in \citet{waudby-smith_estimating_means_2021}. We show that this is a valid e-process in \Cref{sec:Betting} and maintains \(\FDR\) control.

\ifarxiv{\paragraph}{\textbf}{Results} We plot the relative performance of each method to e-BH with UCB sampling in \Cref{fig:OverTime}. For uniform sampling, e-BH and e-variables seem to outperform BH and p-variables, although by a decreasing margin for more arms, especially in the case where \(|\hypset_1| = \lfloor \sqrt{k} \rfloor\). For the UCB sampling algorithm, we see that e-variables and p-variables have relatively similar performance, with the gap narrowing as the number of arms increase as well. Thus, e-variables and e-BH empirically perform on par or better than p-variables with regards to sample complexity. This shows that using e-variables does not require any sacrifice in performance in simple cases where p-variables also work well. Further, e-variables do not require the same \(\log k\) correction that p-variables need for situations where \(X_{1, t}, \dots, X_{k, t}\) are arbitrarily dependent to guarantee \(\FDR\) control at the same level. Thus, e-variables are preferable to p-variables as they are more flexible w.r.t.\ assumptions.
\captionsetup[figure]{skip=0pt}
\begin{figure}[h]
    \captionsetup[subfigure]{aboveskip=-5pt}
    \centering
    \begin{subfigure}[b]{0.32\textwidth}
    \adjustbox{trim=0 0 0 {.2\height},clip}{\includegraphics[width=0.9\textwidth]{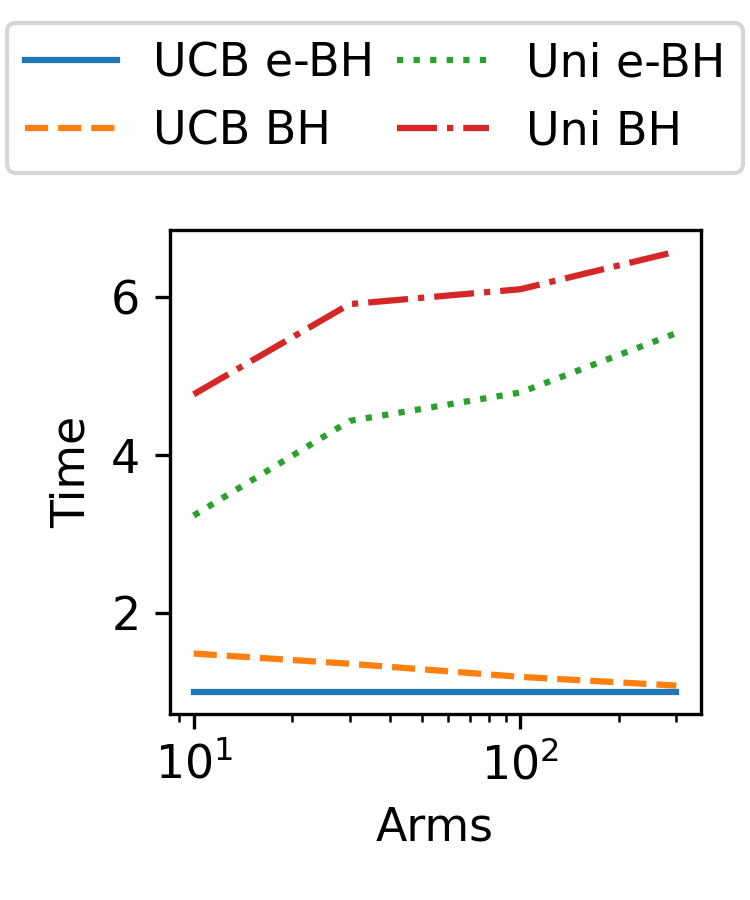}}
    \caption{\(|\hypset_1| = 2\)}
    \end{subfigure}
    \begin{subfigure}[b]{0.32\textwidth}
    \includegraphics[width=0.9\textwidth]{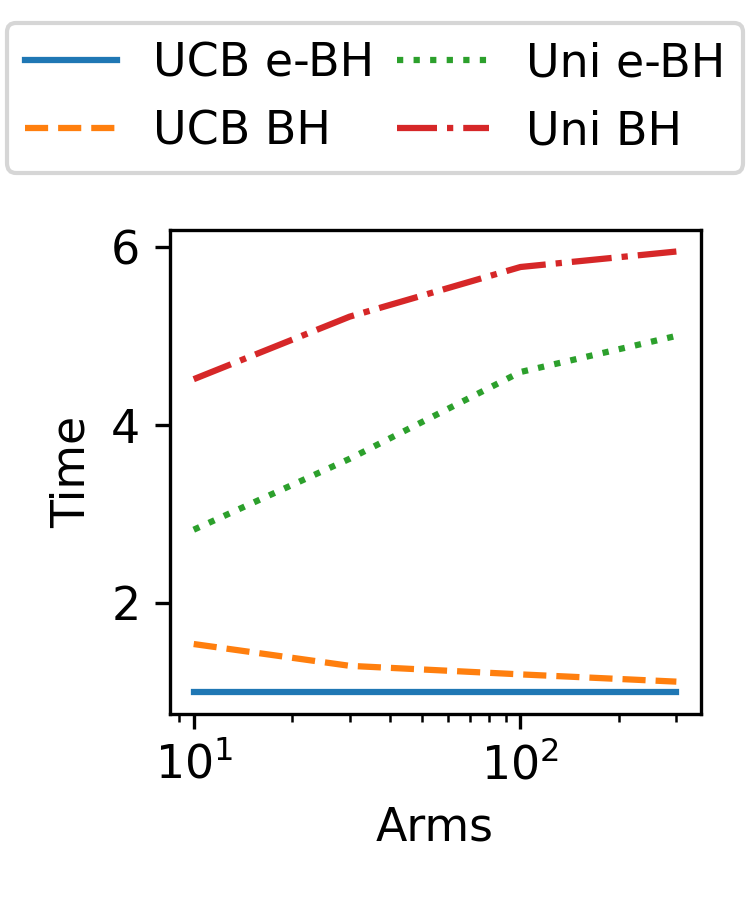}
    \caption{\(|\hypset_1| = \lfloor \log k \rfloor\)}
    \end{subfigure}
    \begin{subfigure}[b]{0.32\textwidth}
    \adjustbox{trim=0 0 0 {.2\height},clip}{\includegraphics[width=0.9\textwidth]{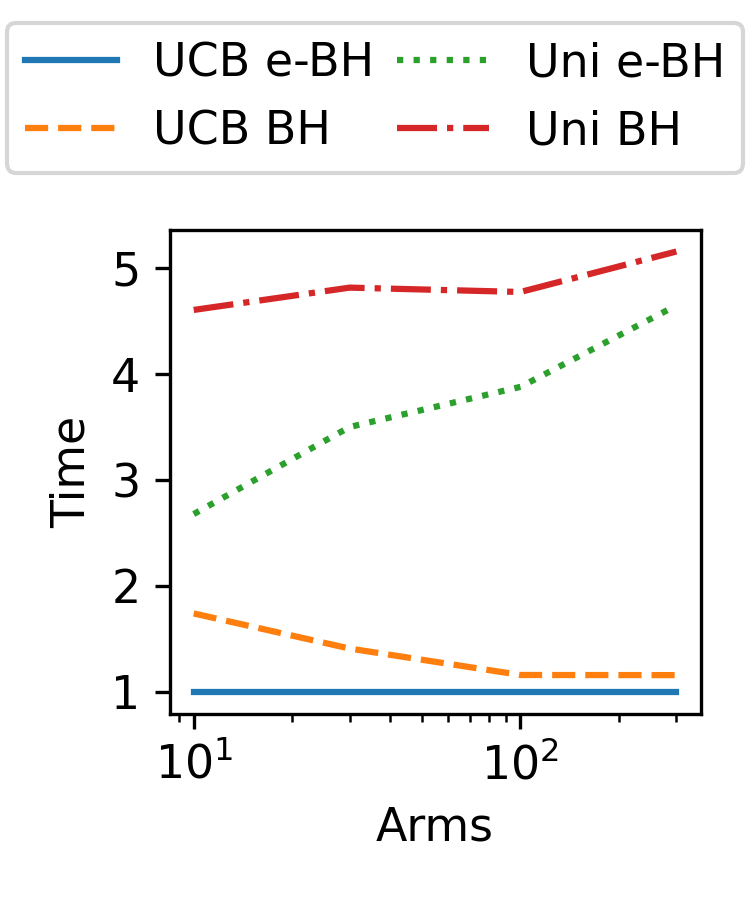}}
    \caption{\(|\hypset_1| = \lfloor \sqrt{k} \rfloor\)}
    \end{subfigure}
    \caption{Relative comparison of time \(t\) to obtain a rejection set, \(\rejset_t\), that has a \(\TPR(\rejset_t) \geq 1 - \delta\) and \(\FDR(\rejset_t) \leq \delta\) where \(\delta=0.05\). This plot compares e-BH vs. BH for both uniform (Uni) and UCB sampling over different numbers of arms (choices of \(k\)) and densities of non-null hypotheses (sizes of \(\hypset_1\)). Time is reported as a ratio to the time taken by UCB e-BH method. Note that the methods using e-variables perform on par or better than methods using p-variables for both sampling strategies.}
    \label{fig:OverTime}
\end{figure}

\ifarxiv{}{\vspace{-5pt}}
\ifarxiv{\section{Conclusion}}{\section{Conclusion, limitations and broader impact}}
\label{sec:Conclusion}
\ifarxiv{}{\vspace{-5pt}}
In this paper, we developed a unified framework for bandit multiple hypothesis testing. 
We demonstrated that applying the e-BH procedure to stopped e-processes guarantees \(\FDR\) control without assumptions on the the dependency between \(X_{1, t}, \dots, X_{k, t}\), exploration strategy, stopping time of the algorithm, ability to query multiple arms, etc. In contrast, existing algorithms using BH and p-variables have \(\FDR\) guarantees that vary with the problem setting and dependence structure among the p-variables. We argued that control of the \(\FDR\) with p-variables can blow up by a factor of \(\log k\), and any p-self-consistent algorithm must decrease its threshold for discovery correspondingly to maintain \(\FDR\) control at the desired level. We provide more detailed explanations of these observations in \Cref{subsec:Structured}.
In addition to demonstrating the generality of our meta-algorithm, we showed that in the standard sub-Gaussian reward setting, the instantiated algorithm matches the sample complexity bounds of the p-variable algorithm by \JJ\ for achieving high \(\TPR\), and has better practical performance than \JJ's algorithm, despite the fact that we improve JJ's guarantees by invoking the self-consistency results of~\citet{su_fdr-linking_2018}. 

The appendices have additional examples of problem settings and simulations that show the utility of e-processes and our general framework. In fact, we can address an even more general setting where the null hypotheses do not have a one-to-one correspondence with the arms; in other words, despite the queries being at the arm-level, the hypotheses being tested could combine arms (for example, comparing different arms). We also discuss the multi-agent setting where there could be multiple agents that operate the same bandit.  We avoided these scenarios in the main paper for simplicity of exposition, since there were enough generalizations to describe in the simpler setup already.

\ifarxiv{}{The main limitation of the work is that it does not develop instance optimal sampling algorithms for multiple testing problem in the described settings with more complicated dependence structures; we believe this is a difficult open problem, requiring specialized techniques in each example. 
We do not foresee any negative societal impact of this work; it is aimed at reducing costs and improving reproducibility in scientific experimentation by controlling false discoveries in adaptive testing.}

\paragraph{Acknowledgments} 
RW acknowledges funding from NSERC RGPIN-2018-03823 and RGPAS-2018-522590.
AR acknowledges funding from NSF DMS 1916320 and ARL IoBT REIGN. Research reported in this paper was sponsored in part by the DEVCOM Army Research Laboratory under Cooperative Agreement W911NF-17-2-0196 (ARL IoBTCRA). The views and conclusions contained in this document are those of the authors and should not be interpreted as representing the official policies, either expressed or implied, of the Army Research Laboratory or the U.S.~Government. The U.S.~Government is authorized to reproduce and distribute reprints for Government purposes notwithstanding any copyright notation herein.

\raggedbottom

\nocite{durrett_probability_theory_2017}
\bibliography{hypothesis}


\appendix
\section{Miscellaneous technicalities}
\label{sec:Misc}

Here, we collect some definitions and properties concerning p-process and supermartingales for the unfamiliar reader, and restate a technical lemma necessary for upcoming proofs.

\begin{lemma}[Lemma 8 from \JJ]
Let \(a \in \reals^n_+\) be a \(n\)-dimensional vector with positive real entries, and for \(i = 1, \dots, n\) let \(Z_i\) be independent random variables where
\begin{align*}
    \prob{Z_i \geq t} \leq \exp(-t / a_i).
\end{align*} Then for any \(\delta \in (0, 1)\),
\begin{align*}
    \sum\limits_{i = 1}^n Z_i \leq 5\log(1 / \delta)\sum\limits_{i = 1}^n a_i.
\end{align*} occurs with at least probability \(1 - \delta\).
\label{lemma:Concentration}
\end{lemma}

\subsection{Equivalence property of p-processes}
\label{sec:PProcesses}

We note the following equivalence proposition for p-processes. Lemma 3 in \citet{howard2021time} and Lemmas 1 and 2 in \citet{ramdas_admissible_2020} makes similar statements regarding sequential processes, but do not additionally characterize the behavior of the infimum of a p-process. 
\begin{proposition}
The following statements are equivalent for a discrete-time process $(P_t)_{t\ge 1}$:
\begin{enumerate}[label = (\roman*)]
    \item   $(P_t)_{t\ge 1}$ is a p-process i.e.\ $\prob{P_\tau\le \alpha}\le \alpha$ for all  (possibly infinite) $\tau \in \mathcal T$  and all $\alpha \in (0, 1)$;
    \item $\prob{P_\tau\le \alpha}\le \alpha$ for all finite $\tau \in \mathcal T$ and all $\alpha \in (0, 1)$;
    \item  $\prob{\exists t\ge 1: P_t \leq \alpha }\leq \alpha$ for all $\alpha \in (0, 1)$;
    \item $\inf_{t\ge 1} P_t $   is superuniformly distributed (its distribution is stochastically larger than uniform).
\end{enumerate}
\label{prop:PEquiv}
\end{proposition}
\begin{proof} In what follows, let $\tau_\alpha\in \mathcal T$ be defined as $\tau_\alpha:= \inf\{t \geq 1: P_t\le \alpha\}$, which is defined to be infinite if $P_t$ never drops below $\alpha$.  

(i)$\Rightarrow$(ii) is trivial by definition. 

(ii)$\Rightarrow$(iii): 
Fix $\alpha\in (0,1)$. 
By (ii), we have  
$\prob{P_{\tau_\alpha\wedge n}\le \alpha} \le \alpha  $ for all $n\ge 1$.
It follows that
\begin{align*} \prob{\exists t\ge 1: P_t \leq \alpha } = \lim_{n\to\infty}  \prob{\exists t\in \{1,\dots,n\}: P_t \leq \alpha } = \lim_{n\to\infty} \prob{P_{\tau_\alpha\wedge n}\le \alpha} \le  \alpha.\end{align*}

 (iii)$\Rightarrow$(iv):  
 For each $\epsilon>0$, since $\inf_{t\ge 1} P_t \le \alpha $ implies $\tau_{\alpha+\epsilon}<\infty$, we have $$\prob{\inf_{t\ge 1} P_t \le \alpha }  \le \prob{ \tau_{\alpha+\epsilon}<\infty} \le \alpha+\epsilon .$$
 As $\epsilon>0$ is arbitrary, we get $\prob{\inf_{t\ge 1} P_t \le \alpha } \le \alpha$, i.e., $\inf_{t\ge 1} P_t $   is superuniformly distributed.

 (iv)$\Rightarrow$(i):  For any $\tau\in \mathcal T$ and $\alpha\in (0,1)$, since $P_{\tau} \ge \inf_{t\ge 1} P_t$, we have 
 $\prob{P_{\tau}\le \alpha}  \le\prob{\inf_{t\ge 1} P_t \le \alpha} \le   \alpha ,$
 thus showing that $(P_t)$ is a p-process.   
\end{proof}
As a direct consequence of Proposition \ref{prop:PEquiv}, 
if $(P_t)$ is a p-process and $P_s$ is uniformly distributed on $[0,1]$ for some $s\ge 1$, then we have 
$
\prob{P_s \le P_t}=1
$ for all $t\ge 1$, since $P_s$ is as small as $\min_{t\ge 1}P_t$.
Therefore, if a p-process does not always take its minimum at a deterministic point $s$,  then
$P_s$ cannot be uniformly distributed on $[0,1]$. In other words, for all deterministic $t$, the random variables $P_t$ are,  in general, not ``precise'' (i.e., uniform on $[0,1]$) p-variables, but conservative ones. 
In contrast, the random variables $E_t$ from an e-process $(E_t)$ are ``precise'' (i.e., have expectation $1$) as soon as $(E_t)$ is a  nonnegative martingale starting at $1$. 

\subsection{Nonnegative supermartingales}
\label{subsec:Supermartingales}
A real-valued process \((M_t)_{t \geq 0}\) is a supermartingale w.r.t.\ a filtration \((\filtration_t)\) if it satisfies:
\begin{align}
    \expect[M_t | \filtration_{t - 1}] \leq M_{t - 1} \text{ for \(t \in \naturals\). }
    \label{eqn:Supermartingale}
\end{align}

For nonnegative supermartingales, we typically assume $M_0=1$ for simplicity; they possess two useful properties. The first is the optional stopping theorem.
\begin{fact}[Optional stopping theorem. \citet{durrett_probability_theory_2017,ramdas_admissible_2020}]
    Let \((M_t)\) be a nonnegative supermartingale w.r.t.\ \((\filtration_t)\). Then, for any stopping time \(\tau \in \Tcal\):
    \begin{align*}
        \expect[M_\tau] \leq M_0.
    \end{align*}
\end{fact}

The second is Ville's inequality.
\begin{fact}[Ville's inequality]
    Let \((M_t)\) be a nonnegative supermartingale w.r.t.\ \((\filtration_t)\). Let \(s \in \reals^+\) be a number in the positive reals.
    \begin{align*}
        \prob{\exists t \in \naturals: M_t \geq s} \leq \frac{M_0}{s}.
    \end{align*}
\end{fact}

\section{Proofs}

\subsection{Proofs of results in \Cref{sec:Dependence}}
\label{sec:DependenceProofs}

The proofs of \Cref{prop:AdaptiveIndPFDR,prop:AdaptiveDepPFDR,prop:AdaptiveStructurePFDR,prop:AdaptiveDepEFDR} all follow from the application of one of \Cref{fact:BHFDR,fact:NewPComplianceFDR,fact:EComplianceFDR}. 

First, we note that \((P_{1, t}), \dots, (P_{k, t})\) being p-processes implies that \(P_{1, t}, \dots, P_{k, t}\) are p-variables for all \(t \in \naturals\). Thus, for any choice of stopping time \(\stoptime \in \Tcal\) for the algorithm, \(P_{1, \stoptime}, \dots, P_{k, \stoptime}\) are p-variables.

Consequently, \Cref{prop:AdaptiveDepPFDR} for the adaptive and dependent p-variables case and \Cref{prop:AdaptiveStructurePFDR} for the adaptive and dependent p-variables with constrained rejection sets case follow from \Cref{fact:BHFDR} and \Cref{fact:NewPComplianceFDR}, respectively.

Similarly, we note that \(E_{1, \stoptime}, \dots E_{k, \stoptime}\) are e-variables, since \((E_{1, t}), \dots, (E_{k, t})\) are e-processes. As a result, \Cref{prop:AdaptiveDepEFDR} follows from \Cref{fact:EComplianceFDR}.

Now, we prove \Cref{prop:AdaptiveIndPFDR} in a slightly different manner than \JJ, using the notion of self-consistency. 

\begin{proof}[Proof of \Cref{prop:AdaptiveIndPFDR}]
Consider an arbitrary \(i \in [k]\). 
\sloppy
Recall that each \(P_{i, t}\) is determined only by \(X_{i, t_i(1)}, \dots, X_{i, t_i(T(i))}\). By independence of \(X_{i, t}\) across \(i \in [k]\) and \(t \in \naturals\), we rename \(X_{i, t_i(j)}\) as  \(X_{i, j}\), since they are identically distributed. Thus, \(P_{i, t}\) is now constructed from \(X_{i, 1}, \dots, X_{i, t}\). We perform this transformation so we can consider \(P_{i, t}\) in the infinite-sample limit. Under our renaming, \(\rejset_{\stoptime}\) is the output of running BH on \(P_{1, T_1(\stoptime)}, \dots, P_{k, T_k(\stoptime)}\). 

Define \(P_i^* \coloneqq \inf_{t \geq 1} P_{i, t}\). Note that \(P_i^*\) are independent p-variables across \(i \in [k]\) by \Cref{prop:PEquiv} since \(X_{1, t}, \dots, X_{k, t}\) are independent for all \(t \in \naturals\) and \((P_{1, t}), \dots, (P_{k, t})\) are p-processes. We can derive self-consistency w.r.t.\ \(P_i^*\) as follows.
\begin{align*}
    \max_{i \in \rejset_{\stoptime}} P_i^* & \leq \max_{i \in \rejset_{\stoptime}} P_{i, T_i(\stoptime)} \tag*{def. of \(P_i^*\)}\\
    & \leq \frac{|\rejset_{\stoptime}|\delta'}{k}. \tag*{p-self-consistency of \(\rejset_{\stoptime}\)}
\end{align*}

Combined with \Cref{fact:NewPComplianceFDR}, we can show that \(\FDR(\rejset_{\stoptime}) \leq \delta'\log(1 + \log(1 / \delta'))\).

Separately, we can also apply the \(\FDR\) guarantee on the output of BH on arbitrarily dependent p-variables from \Cref{fact:BHFDR}. Consequently, we can guarantee \(\FDR(\rejset_{\stoptime}) \leq \delta'\log(1 + \log(1 / \delta')) \wedge \delta' \log k\). Thus, our choice of \(\delta'\) implies \(\FDR(\rejset_{\stoptime}) \leq \delta\), which is our desired result.
\end{proof}

\subsection{Proof of \Cref{prop:DMEProcess}}
\label{subsec:DMEProcessProof}

\citet{howard2021time} actually specifiy a more general form for \(\lambda_{\ell}\) and \(w_\ell\) for the discrete mixture e-process, \(E_{i, t}^{\DM}\). Let \(f\) be a probability density over \((0, \lambda^{\max}]\) and nonincreasing over that interval, \(\overline{\lambda} \in \reals^+\) satisfy \(\overline{\lambda} \leq \lambda^{\max}\),  and \(\eta > 1\) be a step size. \citet{howard2020time} define \(\lambda_{\ell}, w_{\ell}\) as follows:
\begin{align}
    \lambda_\ell \coloneqq \frac{\overline{\lambda}}{\eta^{\ell + 1 / 2}} \text{ and }w_k \coloneqq \frac{\overline{\lambda}(\eta - 1)f(\lambda_{\ell} \sqrt{\eta})}{\eta^{\ell + 1}}.
    \label{eqn:DMGeneralConstants}
\end{align}

Let,
\begin{align*}
    f^{\LIL}_s \coloneqq \frac{(s - 1) s^{s - 1} \ind{0 \leq \lambda \leq 1 / e^s}}{\lambda \log^s \lambda^{-1}},
\end{align*} for any \(s > 1\). We will now connect these definitions to \eqref{eqn:DMConstants}. Set \(\overline{\lambda} = 1 / e\), \(\eta = e\), and \(f = f_{2}^{\LIL}\). Then,
\begin{align*}
    \lambda_{\ell} &= \frac{1}{e^{\ell + 3/2}},\\
    w_{\ell} &= \frac{\frac{1}{e}(e - 1) f_2^{\LIL}(\frac{1}{e^{\ell + 3/2}} \cdot \sqrt{e})}{e^{\ell + 1}}\\
    &= \frac{(e - 1)f_2^{\LIL}(\frac{1}{e^{\ell + 1}})}{e^{\ell + 2}}\\
    &= \frac{(e - 1)\frac{2\ind{\ell \geq 1}}{\frac{1}{e^{\ell + 1}} \log^2(e^{\ell + 1})}}{e^{\ell + 2}}\\
    &=\frac{2(e - 1)\ind{\ell \geq 1}}{(\frac{1}{e^{\ell + 1}})(e^{\ell + 2})\log^2(e^{\ell + 1})}\\
    &=\frac{2(e - 1)\ind{\ell \geq 1}}{e(\ell + 1)^2}.
\end{align*} By reindexing \(\ell\), we can redefine the variables as follows:
\begin{align*}
    \lambda_\ell = \frac{1}{e^{\ell + 5 / 2}}
    \text{ and }w_{\ell} = \frac{2(e - 1)}{e(\ell + 2)^2}.
\end{align*}

To prove \Cref{prop:DMEProcess}, we prove the following more general proposition which is derived from existing results in \citet{howard2021time}.
\begin{proposition}[Derived from equations (49) and (82) of \citet{howard2021time}]
Let,
\begin{align*}
    E_{i, t} \coloneqq \sum\limits_{\ell = 0}^\infty w_{\ell} \exp\left(\sum\limits_{j = 1}^{T_i(t)}\lambda_{\ell}(X_{i, t_i(j)} - \mu_0) - \frac{\lambda_{\ell}^2}{2}\right).
\end{align*} If \(\sum\limits_{\ell = 0}^\infty w_{\ell}\leq 1\), then \((E_{i, t})\) is a nonnegative supermartingale, and consequently an e-process, if the conditional distribution \(X_{i ,t} \mid \filtration_{t - 1}\) is \(1\)-sub-Gaussian and \(\expect[X_{i, t} \mid \filtration_{t - 1}] \leq \mu_0\) for all \(t \in \naturals\).
\label{prop:DMIsNM}
\end{proposition}
\begin{proof}
Let
\begin{align*}
M_{i, t}^\lambda \coloneqq \exp\left(\sum\limits_{j = 1}^{T_i(t)}\lambda(X_{i, t_i(j)} - \mu_0) - \frac{\lambda^2}{2}\right),
\end{align*} where \(\lambda \in \reals\). \((M_{i, t})\) is a nonnegative supermartingale because of the sub-Gaussian and bounded conditional mean assumptions on \(X_{i, t}\). Let, \(w_{\mathrm{sum}} = \sum\limits_{\ell = 0}^\infty w_{\ell}\). Now, we show that \(E_{i, t}\) is a supermartingale:
\begin{align*}
    \expect\left[E_{i, t} \mid \filtration_{t - 1}\right] &= \expect\left[\sum\limits_{\ell = 0}^\infty w_{\ell} M_{i, t}^{\lambda_{\ell}} \mid \filtration_{t - 1}\right]\\
    &=\sum\limits_{\ell = 0}^\infty w_{\ell} \expect\left[M_{i, t}^{\lambda_{\ell}} \mid \filtration_{t - 1}\right]\\
    &\leq \sum\limits_{\ell = 0}^\infty w_{\ell} M_{i, t - 1}^{\lambda_{\ell}} \\
    &=E_{i, t -1}.
\end{align*} The sole inequality is by the supermartingale property of \((M_{i, t}^{\lambda_{\ell}})\). Thus, we have shown our desired result.
\end{proof}

\subsection{Proof of \Cref{thm:DiscreteESampleComplexity}}
\label{subsec:SampleComplexityProof}

We follow a similar path as the sample complexity proof (i.e. Theorem 2) from \JJ\ for \Cref{thm:DiscreteESampleComplexity}. Our goal is to show that we reject the following set with at least \(1 - \delta\) probability:
\begin{align}
    \captureset = \{i \in \hypset_1: \widehat{\mu}_{i, t} + \varphi(t, \delta) \geq \mu_i \text{ for all }t \in \naturals\}.
\end{align}
\begin{lemma}
    \(\expect[|\captureset|] \geq (1 - \delta)|\hypset_1|\).
    \label{lemma:CaptureSetExpectation}
\end{lemma}

\begin{proof}
    We have the following:
    \begin{align*}
    \expect[|\captureset|] &= \sum\limits_{i \in \hypset_1} \prob{\widehat{\mu}_{i, t} + \varphi(t, \delta) \geq \mu_i }\\
    &\geq \sum\limits_{i \in \hypset_1} \prob{|\widehat{\mu}_{i, t} - \mu_i| \leq \varphi(t, \delta)}\\  
    &\geq (1 - \delta)|\hypset_1| \tag*{\Cref{fact:LILBound}}.
    \end{align*}
\end{proof}

\Cref{lemma:CaptureSetExpectation} shows that rejecting \(\captureset\) is sufficient to produce rejection sets that have \(\TPR(\rejset) \geq 1 - \delta\). Thus, our goal in this proof is to show a bound on \(T \coloneqq \min\{t \in \naturals: \captureset \subseteq \rejset_{t}\}\) with at least \(1 - \delta\) probability, where \(T = \infty\) if \(\captureset \not\subseteq \rejset_t\) for all \(t \in \naturals\). Note that for all \(t \geq T\), \(\rejset_T \subseteq \rejset_t\) by the way \eqref{eqn:UCB} is defined --- it does not sample arms that have already been rejected. 

We note that we can use any \(\varphi\) defined in \Cref{fact:LILBound} in this proof and still achieve the desired result. For simplicity, we use \(\varphi\) to denote \(\varphi^0\) in this proof. First we define a notion of inverse for \(\varphi\). Let 
\begin{align}
\varphi^{-1}(\epsilon, \delta) \coloneqq \min\{t: \varphi(t, \delta) \leq \epsilon\}.
\end{align} \JJ\ and other work \citep{jamieson_lil_ucb_2014} show that for some absolute constant \(c > 0\),
\begin{align}
\varphi^{-1}(\epsilon, \delta) \leq c\epsilon^{-2}\log(\log(\epsilon^{-2}) / \delta)\ \text{for all}\ \epsilon \in \reals^+, \delta \in (0, 1).
\label{eqn:InvPhiBound}
\end{align} Also, recall that \(f \lesssim g\) denotes \(f\) asymptotically dominates \(g\) i.e.\ there exist \(c > 0\) that is independent of the problem parameters such that \(f \leq cg\).

We decompose \(T\) into the number of time steps the algorithm samples a null arm, and the number of time steps the algorithm samples a non-null arm:
\begin{align}
    T = \sum\limits_{t = 1}^\infty \ind{\captureset \not\subseteq \rejset_t} = \sum\limits_{t = 1}^\infty \ind{I_t \in \hypset_0, \captureset \not\subseteq \rejset_t} + \ind{I_t \in \hypset_1, \captureset \not\subseteq \rejset_{t}}.
\end{align}

Our first goal is to prove a sample complexity bound on \(\sum\limits_{t = 1}^\infty \ind{I_t \in \hypset_0, \captureset \not\subseteq \rejset_t}\).  We define the following variables for each \(i \in [k]\). 
\begin{align}
    \rho_i \coloneqq \inf \{\rho \in [0, 1]: |\widehat{\mu}_{i, t} - \mu_i| > \varphi(t, \rho) \text{ for all }t\in \naturals\} \cup \{1\}.
\end{align}

\begin{lemma}
For each \(i \in [k]\), \(\prob{\rho_i \leq s}\leq s\) for \(s \in (0, 1)\) i.e.\ \(\rho_i\) is superuniformly distributed.
    \label{lemma:SuperuniformRho}
\end{lemma}
The above lemma follows directly from \Cref{fact:LILBound}. We also define a concentration bound for independent superuniformly distributed variables.
\begin{lemma}
For any fixed positive reals $a_1,\dots,a_d$, independent superuniformly distributed random variables $r_1,\dots,r_d$, and \(\beta \in (0, 1)\), the following event occurs with probability at least \(1 - \beta\):
\begin{align*}
    \sum\limits_{i = 1}^d a_i \log(1 / r_i) \leq 5 \log(1 / \beta) \sum\limits_{i = 1}^d a_i.
\end{align*}
\label{lemma:SuperuniformConcentration}
\end{lemma}
This lemma follows directly from recognizing \(a_i\log(1 / r_i)\) satisfies the requirements for \(Z_i\) in \Cref{lemma:Concentration}. Now, we will show that the UCB for each \(i \in \captureset\) will be above \(\mu_i\).  
\begin{lemma}
Let \(\nu_i\) be sub-Gaussian for each \(i \in [k]\). Any algorithm with \((\EC_t)\) that outputs \(\sampleset_t = \{I_t\}\) as defined in \eqref{eqn:UCB} has the following property:
\begin{align*}
    \sum\limits_{t = 1}^{\infty}\ind{I_t \in \hypset_0, \captureset \not\subseteq \rejset_t} \lesssim \sum\limits_{i \in \hypset_0} \gap_i^{-2}\log(\log(\gap_i^{-2}) / \delta\rho_i).
\end{align*}
\label{lemma:NullSampleComplexity}
\end{lemma}

\begin{proof}
The following is true for any \(i \in \captureset\) and \(t \in \naturals\):
\begin{align*}
    \widehat{\mu}_{i, t} + \varphi(T_i(t), \delta)) &\geq \mu_i- \varphi(T_i(t), \rho_i) + \varphi(T_i(t), \delta))  \tag*{by def. of \(\rho_i\) and \(\captureset\)}\\
    &\geq \mu_i \tag*{by def. of \(\captureset\)}.
\end{align*}

Thus, \(\{\captureset \not\subseteq \rejset_t\}\) implies for any \(t\in \naturals\):
\begin{align*}
    \argmax_{i \in [k] \setminus \rejset_t} \widehat{\mu}_{i, t} + \varphi(T_i(t), \delta) & \labelrel\geq{eqn:SCLowerBound} \min_{i \in \captureset} \mu_i\\
    & \geq \min_{i \in \hypset_1} \mu_i,  \numberthis \label{eqn:UCBLargerThanMean}
\end{align*}
where inequality \eqref{eqn:SCLowerBound} is by the definition of \(\captureset\).

In addition, we argue that the UCB for \(i \in \hypset_0\) will shrink below \(\min_{i \in \hypset_1} \mu_i\) quickly. For \(i \in \hypset_0\), the following is true for any \(t \in \naturals\):
\begin{align*}
    \widehat{\mu}_{i, t} + \varphi(T_i(t), \delta) &\leq \mu_i + \varphi(T_i(t), \rho_i) + \varphi(T_i(t), \delta)\\
    &\leq \mu_i + 2\varphi(T_i(t), \delta\rho_i). \numberthis \label{eqn:NullUCBBound}
\end{align*}

Thus, \(\{\forall i \in \hypset_0: \mu_i + 2\varphi(T_i(t), \delta\rho_i) \leq \min_{i \in \hypset_1} \mu_i,\ \captureset \not\subseteq \rejset_t\} \implies \{I_t \in \hypset_1\}\) for all \(t \in \naturals\) by \eqref{eqn:UCBLargerThanMean} and \eqref{eqn:NullUCBBound}.

Subsequently, we argue the following:
\begin{align*}
    \sum\limits_{t = 1}^{\infty}\ind{I_t \in \hypset_0, \captureset \not\subseteq \rejset_t} & \leq \sum\limits_{t = 1}^{\infty}\ind{\exists i \in \hypset_0: \mu_i + 2\varphi(T_i(t), \delta\rho_i) > \min_{i \in \hypset_1} \mu_i, \captureset \not\subseteq \rejset_t}\\
    & \leq \sum\limits_{i \in \hypset_0} \varphi^{-1}(\gap_i / 2, \delta\rho_i) \tag*{\(\mu_i \leq \mu_0\) for all \(i \in \hypset_0\)}\\
    &\lesssim \sum\limits_{i \in \hypset_0} \gap_i^{-2}\log(\log(\gap_i^{-2}) / \delta\rho_i).
\end{align*}

Thus, we have shown our desired result.
\end{proof}

Now, we proceed to show a bound on \(\sum\limits_{t = 1}^\infty \ind{I_t \in \hypset_1, \captureset \not\subseteq \rejset_t}\). Denote \(\pi\) as an arbitrary mapping from \(\hypset_1\) to \([|\hypset_1|]\). Let \((x)_+ = x \vee 0\) for any \(x \in \reals\). We define additional variables as follows:
\begin{align*}
    \ell'_{i} &\coloneqq (\lceil \log(2\gap_i^{-1}) - 5 / 2 \rceil)_+,\\
    \rho_i^{\DM} &\coloneqq \min_{t \in \naturals}\ \frac{1}{\exp\left(\sum\limits_{j = 1}^{T_i(t)} \lambda_{\ell_i'}(\mu_i - X_{i, t_i(j)}) - \lambda_{\ell_i'}^2 / 2\right)}.
\end{align*}
\begin{lemma}
\(\prob{\rho_i^{\DM} \leq s} \leq s\) for \(s \in (0, 1)\) i.e.\ \(\rho_i^{\DM}\) is superuniformly distributed for each \(i \in \hypset_1\).
\label{lemma:DMSuperuniform}
\end{lemma}
\begin{proof}
First, we prove an underlying process is a nonnegative supermartingale.  Let \begin{align*}
    M_t = \exp\left(\sum\limits_{j = 1}^{T_i(t)} \lambda_{\ell_i'}(\mu_i - X_{i, t_i(j)}) - \lambda_{\ell_i'}^2 / 2\right).
\end{align*}
Assume that arm \(i\) is selected at time \(t\) --- otherwise the supermartingale property is directly satisfied.
\begin{align*}
    \expect[M_t \mid \filtration_{t - 1}] &= \expect\left[\exp\left(\sum\limits_{j = 1}^{T_i(t)} \lambda_{\ell_i'}(\mu_i - X_{i, t_i(j)}) - \lambda_{\ell_i'}^2 / 2\right) \mid \filtration_{t - 1}\right]\\
    &= \expect\left[\exp\left(\lambda_{\ell_i'}(\mu_i - X_{i, t}) - \lambda_{\ell_i'}^2 / 2\right) \mid \filtration_{t - 1}\right]  \exp\left(\sum\limits_{j = 1}^{T_i(t - 1)} \lambda_{\ell_i'}(\mu_i - X_{i, t_i(j)}) - \lambda_{\ell_i'}^2 / 2\right)\\
    & \leq M_{t - 1},  
\end{align*} where the final inequality holds because \(X_{i, t}\) are i.i.d.\ across \(t \in \naturals\), have mean \(\mu_i\), and are \(1\)-sub-Gaussian.

Thus, \(\rho_i^{\DM}\) is a superuniform random variable by applying Ville's inequality to \((M_t)\).

\end{proof}

\begin{proposition}[Growth of \(E_{i, t}^{\DM}\)]
When \(i \in \hypset_1\) and \(\nu_i\) is \(1\)-sub-Gaussian,
\begin{align*}
    \log E_{i ,t} \gtrsim \gap_i^2T_i(t) - \log \log (\gap_i^{-2}) - \log(1 / \rho_i^{\DM})
.\end{align*}
\label{prop:DMGrowth}
\end{proposition}
\begin{proof}
We show the following lower bound on \(E_{i, t}^{\DM}\):
\begin{align*} 
    E_{i, t}^{\DM} &= \sum\limits_{\ell = 0}^\infty w_{\ell}\exp(T_i(t)(\lambda_{\ell}\gap_i - \lambda_{\ell}^2))\exp\left(\sum\limits_{j = 1}^{T_i(t)} \lambda_\ell^2 / 2 - \lambda_\ell(\mu_i - X_{i, t_i(j)})\right)\\
    & \geq w_{\ell_i'}\exp(T_i(t)(\lambda_{\ell_i'}\gap_i - \lambda_{\ell_i'}^2))\exp\left(\sum\limits_{j = 1}^{T_i(t)} \lambda_{\ell_i'}^2 / 2 - \lambda_{\ell_i'}(\mu_i - X_{i, t_i(j)})\right)\\
    & \geq \exp\left(\frac{1}{4e}\gap_i^2T_i(t) - \log(1 / w_{\ell_i'}) - \log(1 / \rho_i^{\DM})\right) \tag*{by def. of \(\ell_i'\) and \(\rho_i^{\DM}\)}.
\end{align*} Thus, plugging in \(w_{\ell_i'}\), we get our desired result.
\end{proof}

\begin{proof}[Proof of \Cref{thm:DiscreteESampleComplexity}]

By \Cref{prop:DMGrowth},
\begin{align*}
    T_i(t) \gtrsim \gap_i^{-2}\log(\varepsilon\log(\gap_i^{-2})/\rho_i^{\DM}) 
\end{align*} implies \(E_{i, t} \geq \varepsilon\) for \(\varepsilon > 0\).

Now, we can derive the following bound:
\begin{align*}
    \sum\limits_{t = 1}^\infty \ind{\captureset \not\subseteq \rejset_t} =& \sum\limits_{t = 1}^\infty \ind{I_t \in \hypset_0, \captureset \not\subseteq \rejset_t} + \sum\limits_{t = 1}^\infty \ind{I_t \in \hypset_1, \captureset \not\subseteq \rejset_t}\\
    \lesssim &\sum\limits_{i \in \hypset_0}\gap_i^{-2}\log\log(\gap_i^{-2}) + \gap_i^{-2}\log(1 / \rho_i) + \gap_i^{-2}\log(1 / \delta) \\
    &+ \max_{\pi}\sum\limits_{i \in \hypset_1}\gap_i^{-2}\log\log(\gap_i^{-2}) + \gap_i^{-2}\log(1 / \rho_i^{\DM}) + \gap_i^{-2}\log(k / \pi(i)\delta). \tag*{by \Cref{lemma:NullSampleComplexity}}
\end{align*}

Recall that \(\rho_i\), by \Cref{lemma:SuperuniformRho}, and \(\rho_i^{\DM}\), by \Cref{lemma:DMSuperuniform}, are superuniform random variables that are independent across \(i \in [k]\) and \(i \in \hypset_1\), respectively. Consequently, we can apply \Cref{lemma:SuperuniformConcentration} at level \(\beta = \delta / 2\) to \(\rho_i\) for \(i \in [k]\) and \(\rho_i^{\DM}\) for \(i \in \hypset_1\). Then, the following happens with at least \(1 - \delta\) probability:
\begin{align*}
    \sum\limits_{t = 1}^\infty \ind{\captureset \not\subseteq \rejset_t}
    \lesssim& \sum\limits_{i \in \hypset_0}\gap_i^{-2}\log(\log(\gap_i^{-2})/\delta)\\
    &+ \max_{\pi}\sum\limits_{i \in \hypset_1}\gap_i^{-2}\log(\log(\gap_i^{-2})/\delta) + \gap_i^{-2}\log(k / \pi(i)).
\end{align*} We can derive two different bounds. The first is using the fact that \(\sum\limits_{i = 1}^{|\hypset_1|} \log(k / i) \leq k\). As a result,
\begin{align*}
    \sum\limits_{t = 1}^\infty \ind{\captureset \not\subseteq \rejset_t} \lesssim k\gap^{-2}\log(\log(\gap^{-2})/\delta).
\end{align*} The second comes from dropping the \(\pi(i)\) term, which is as follows:
\begin{align*}
    \sum\limits_{t = 1}^\infty \ind{\captureset \not\subseteq \rejset_t} \lesssim \sum\limits_{i \in \hypset_0}\gap_i^{-2}\log(\log(\gap_i^{-2})/\delta) + \sum\limits_{i \in \hypset_1}\gap_i^{-2}\log(k\log(\gap_i^{-2})/\delta).
\end{align*} Thus, we have shown both sample complexity bounds as desired.
\end{proof}

\section{Generic algorithms for \((\EC_t)\)}
\label{sec:GenericAlgs}
We propose two generic algorithms that can be used for the exploration component in \Cref{alg:Framework} regardless of the type of hypotheses tested or what the joint distribution of \(X_{1, t}, \dots, X_{k, t}\) is. For simplicity, we assume that the algorithm can always sample each arm separately, i.e.\ \(\{\{1\}, \dots, \{k\}\} \subseteq \Kcal\).

\begin{algorithm}[t]
\caption{An generic algorithm that uses a BAI subroutine to find a best arm that is not in the rejection set for the algorithm to then repeatedly sample and eventually reject.}
\KwIn{A BAI algorithm \(\Bcal\) that takes in \(B \subseteq [k]\), and a history of samples and initial randomness \(D_t(B) \coloneqq U \cup \{(i, j, X_{i, j}): j \leq t, i \in \sampleset_j \cap B\}\). At each step, \(\Bcal\) outputs a superarm \(\sampleset \in \Kcal\) to sample next, or a best arm \(i \in B\). Let \(\delta \in (0, 1)\) be the level of \(\FDR\) control and \(\delta'\in (0, 1)\) be the corrected level for p-variables. Let \((e_{i, t})\) and \((p_{i, t})\) be realized values of e-processes and p-processes, respectively, for each \(i \in [k]\). Let \(\stoptime \in \Tcal\) be the stopping time for the algorithm.}

\textbf{Initialize} \(\rejset_0 \coloneqq \emptyset\)\\
\textbf{Initialize} \(\mathrm{best arm}\) \(\coloneqq\) \textbf{none}\\
\For{\(t \in 1, \dots, \)}{
    \(B \coloneqq [k] \setminus \rejset_{t -1}\)\\
    \uIf{\(\mathrm{bestarm}\) is \textbf{none} or \(\mathrm{bestarm} \in \rejset_{t - 1}\)}{
    \(\sampleset_t  \coloneqq \Bcal(B, D_{t - 1}(B))\)\\
    \lIf{\(\Bcal(B, D_{t - 1}(B))\) terminated with best arm \(I_t\)}{\(\mathrm{bestarm} \coloneqq I_t\)}
    }\Else{
        \(\sampleset_t \coloneqq \{\mathrm{bestarm}\}\) (or an arbitrary \(\sampleset \in \Kcal\) such that \(\mathrm{bestarm} \in \Kcal\)).
    }
    Update e-process or p-process for each queried arm not in \(\rejset_{t - 1}\).\\
    \(\rejset_t \coloneqq 
    \begin{cases}
    \BH[\delta'](p_{1, t}, \dots, p_{k, t}) \text{ or arbitrary p-self-consistent set} & \text{if using p-variables}\\
    \EBH[\delta](e_{1, t}, \dots, e_{k, t}) \text{ or arbitrary e-self-consistent set}& \text{if using e-variables}
    \end{cases}\)\\
    \lIf{\(\stoptime = t\)}{stop and \Return \(\rejset_t\)}
}
\label{alg:BAI}
\end{algorithm}

\paragraph{Reduction to best arm identification (BAI)} The first relies on having access to a best arm identification (BAI) algorithm. BAI is well studied problem, and there exist many algorithms for it in both the standard bandit setting \citep{audibert2010best,kalyanakrishnan_pac_subset_2012,jamieson_lil_ucb_2014,kaufmann_complexity_bestarm_2016} and combinatorial bandit settings \citep{chen_combinatorial_multiarmed_2016, chen_nearly_optimal_2017a,jourdan_efficient_pure_2021a}. A BAI algorithm returns the ``best arm'' i.e.\ the arm with the highest mean reward, with high probability. Thus, we can employ a BAI algorithm as a subroutine to repeatedly find the best arm out of arms not in the rejection set, and then repeatedly sample that best arm until it is rejected. \Cref{alg:BAI} formulates an algorithm using a BAI subroutine that fits the meta-algorithm introduced in \Cref{alg:Framework} . Consequently, we can immediately have access to algorithms for multiple testing that have non-trivial exploration components for a wide variety of settings.

\begin{algorithm}[t]
\caption{An generic algorithm applicable to any combinatorial bandit and set of hypotheses that utilizes the evidence itself (i.e. p-variables or e-variables) to select arms to sample.}
\KwIn{Let \(\delta \in (0, 1)\) be the level of \(\FDR\) control and \(\delta'\in (0, 1)\) be the corrected level for p-variables. \((e_{i, t})\) and \((p_{i, t})\) be realized values of e-processes and p-processes, respectively, for each \(i \in [k]\). Let \(\stoptime \in \Tcal\) be the stopping time for the algorithm.}

\textbf{Initialize} \(\rejset_0 \coloneqq \emptyset\)\\
\textbf{Initialize} \(e_{i, 0} = 1\) or \(p_{i, 0} = 1\) for all \(i \in [k]\)\\
\For{\(t \in 1, \dots, \)}{
    \uIf{\(t \leq k\)}{
        \(I_t \coloneqq t\)\\
        \(\sampleset_t \coloneqq \{I_t\}\) or an arbitrary \(\sampleset \in \Kcal\) where \(I_t \in \sampleset\)
    } \Else{
    \(I_t \in \begin{cases}
    \argmin_{i \in [k] \setminus \rejset_{t - 1}} p_{i, t - 1} & \text{if using p-variables}\\
    \argmax_{i \in [k]\setminus \rejset_{t - 1}} e_{i, t - 1} & \text{if using e-variables}
    \end{cases}\) (an arbitrary element of \(\argmin\)/\(\argmax\)).\\
    \(\sampleset_t \coloneqq \{I_t\}\) or an arbitrary \(\sampleset \in \Kcal\) where \(I_t \in \sampleset\)
    }
    Update e-process or p-process for each queried arm not in \(\rejset_{t - 1} \).\\
    \(\rejset_t \coloneqq 
    \begin{cases}
    \BH[\delta'](p_{1, t}, \dots, p_{k, t}) \text{ or arbitrary p-self-consistent set} & \text{if using p-variables}\\
    \EBH[\delta](e_{1, t}, \dots, e_{k, t}) \text{ or arbitrary e-self-consistent set}& \text{if using e-variables}
    \end{cases}\)\\
    \lIf{\(\stoptime = t\)}{stop and \Return \(\rejset_t\)}
}
\label{alg:BestEvidence}
\end{algorithm}

\paragraph{Largest e-process (or smallest p-process)} If no apparent exploration strategy exists, we can always select the arm that currently has the most evidence for rejection, but has not yet been rejected. \Cref{alg:BestEvidence} illustrates this algorithm --- our exploration strategy is to simply pick the superarms that contains the arm that already has the ``most'' evidence (largest e-value or smallest p-value). Thus, simply having e-variables or p-variables for the hypotheses we are testing can be used to inform the sampling strategy.

Both of the aforementioned algorithms guarantee \(\FDR\) control due to being instances of \Cref{alg:Framework}.
\begin{proposition}
\Cref{alg:BAI,alg:BestEvidence} guarantee that \(\sup_{\tau \in \Tcal}\FDR(\rejset_\tau) \leq \delta\).
\end{proposition}
As a result, we always have a default choice of exploration component if we are unaware of any domain specific strategies for sampling.
\section{Extensions on the bandit setting}

In this section, we consider some special cases and extensions on the bandit settings. This includes settings involving streaming data, constrained rejection sets, multiple agents, and hypotheses involving multiples arms. Critically, we show how our framework can be easily adaptable to each of these settings to still maintain valid \(\FDR\) guarantees. 
\subsection{Streaming data setting}
\label{subsec:StreamingData}

A unique instance of the combinatorial bandits is the streaming data setting, where the algorithm has access to the all rewards at each time step. Instead of choosing a sampling policy, the algorithm can choose a stopping time \(\tau_i\) for each arm \(i \in [k]\) that marks when the algorithm will cease observing arm \(i\). Although \(X_{1, t}, \dots, X_{k, t}\) may be arbitrarily dependent, \Cref{alg:Framework} with e-variables can still use all the observations from each arm at each time step. This is because e-BH on e-variables maintains the same \(\FDR\) control irrespective of dependence structure. Thus, we can propose a simple strategy in \Cref{alg:Streaming} that stops the monitoring of an arm once that arm has been rejected by e-BH, and can limit the amount of time between rejections or total time run before the algorithm stops. By \Cref{prop:AdaptiveDepEFDR}, we have the following \(\FDR\) guarantee.
\begin{proposition}
\Cref{alg:Streaming} ensures \(\sup_{\tau \in \Tcal}\FDR(\rejset_\tau^*) \leq \delta\).
\label{prop:StreamingFDR}
\end{proposition}

\citet{bartroff_sequential_tests_2020a} also study multiple testing in the streaming data setting, and prove \(\FDR\) guarantees similar to \Cref{prop:StreamingFDR} for an algorithm that is virtually identical to \Cref{alg:Streaming} with p-variables and BH. A key difference between their results and ours is that they use test statistics in their algorithm instead of p-variables, and make assumptions about the power of the test statistics that also allow them to provide guarantees about the false negative rate i.e.\ the expected proportion of hypotheses that are not rejected which are true discoveries. Thus, our framework for \(\FDR\) control subsumes existing methods for the streaming setting. Other error metrics such as family-wise error rate and probabilistic bounds on the \(\FDP\) have also been studied in the sequential setting \citep{bartroff_sequential_tests_2014,bartroff_rejection_principle_2016,bartroff_multiple_hypothesis_2017a}.

\begin{algorithm}
\label{alg:Streaming}
\caption{An algorithm for monitoring in the streaming data setting. This algorithm stops when the maximum time \(t_{\max}\) has been reached, or more than \(t_{\mathrm{gap}}\) steps have passed since the last rejection. Once an arm is added to \(\rejset_t\), the algorithm stops monitoring it.}
\(\rejset_0 = \emptyset\)\\
\(t_{\mathrm{prev. rejection}} = 0\)\\
\For{\(t \in 1, \dots, t_{\max}\)}{
    \(\sampleset_t \coloneqq [k] \setminus \rejset_{t - 1}\)\\
    \(\rejset_t \coloneqq \begin{cases}
    \EBH[\delta](e_{1, t}, \dots, e_{k, t}) & \text{if using e-variables}\\
    \BH[\delta'](p_{1, t}, \dots, p_{k, t}) & \text{if using p-variables}\\
    \end{cases}\)\\
    
    \lIf{\(t - t_{\mathrm{prev. rejection}} > t_{\mathrm{gap}}\) \textbf{or} \(\rejset_t = [k]\)}{\Return \(\rejset_t\)} 
}
\Return \(\rejset_t\)
\end{algorithm}

\subsection{Structured rejection sets}
\label{subsec:Structured}

Structured rejection sets arise in problems where there is a fixed hierarchy that restrict the sets of hypotheses that can be rejected e.g.\ hypothesis \(2\) can only be rejected if hypothesis \(1\) is rejected also. Recent work in multiple testing with \(\FDR\) control has studied settings with general structural constraints \citep{lei_general_interactive_2020} and when the constraints have been restricted to form a directed acyclic graph (DAG) \citep{ramdas_sequential_algorithm_2019}. A DAG constraint requires all predecessors of a hypothesis in the DAG to be rejected before the hypothesis itself can be rejected. Thus, the algorithm does not necessarily output the result of BH or e-BH, but rather a p-self-consistent or e-self-consistent set, respectively. \Cref{table:StructuredDependence} illustrates the \(\FDR\) guarantees for p-variables in the structured setting for different dependence relationsips between \(X_{1, t}, \dots, X_{k, t}\). 
In case with adaptive \((\EC_t)\) and \(\stoptime\) when \(X_{1, t} \dots, X_{k, t}\) are independent --- the guarantee in that setting remains unchanged due to the proof of \(\FDR\) control already being based upon the fact that the output of BH was p-self-consistent to \(P_1^*, \dots, P_k^*\). Similarly, e-variables still do not pay a penalty when moving from e-BH to an arbitrary e-self-consistent set. The \(\FDR\) when using e-variables remains below \(\delta\) after setting \(\alpha = \delta\).

\begin{table}[h!]
    \vspace{10pt}

    \caption{\(\FDR\) control guaranteed by an arbitrary p-self-consistent set, and the \(\delta'\) to ensure \(\delta\) control of \(\FDR\) in \Cref{alg:Framework} under different dependence structures and adaptivity of \((\EC_t)\). Adaptive and non-adaptive strategies no longer have different guarantees when outputting a p-self-consistent set. On the other hand, the \(\FDR\) control of an e-self-consistent set remains unchanged at \(\alpha = \delta\).}
    \centering
	\begin{tabular}{|c|c|l|}
\hline
\multicolumn{1}{|l|}{\textit{}}                     & \multicolumn{2}{c|}{\textbf{Dependence of} \(X_{1, t}, \dots, X_{k, t}\)}                                                                                                     \\ \hline
\textbf{Adaptivity of}                     & \multirow{2}{*}{\textit{independent}}                                                               & \multicolumn{1}{c|}{\multirow{2}{*}{\textit{arbitrarily dependent}}}    \\
\multicolumn{1}{|l|}{\((\EC_t)\) and \(\stoptime\)} &                                                                                                     & \multicolumn{1}{c|}{}                                                   \\ \hline
\textit{adaptive}                                   & \multicolumn{1}{l|}{$\FDR(\rejset) \leq \alpha((1 + \log(1 / \alpha)) \wedge \log k)$}                & $\FDR(\rejset) \leq \alpha \log k $ \\ \cline{1-1}
\textit{non-adpative}                               & \multicolumn{1}{l|}{\(\delta' = c_\delta \vee (\delta / \log k)\) (Prop.~\ref{prop:AdaptiveIndPFDR})} & \(\delta' = c_\delta / \log k\)                                         \\ \hline
\end{tabular}
    \vspace{10pt}
    \label{table:StructuredDependence}
\end{table}

We show an example in \Cref{fig:DAG} of a set of hypotheses in a DAG structure. Thus, a hypothesis can only be rejected if its predecessors in the DAG are also rejected. We compare the output of BH, e-BH, and  both the largest e-self-consistent set and p-self-consistent set that respect the DAG constraints. The e-values are calculated assuming that the p-variables are reciprocals of e-variables. We assume the p-variables are all arbitrarily dependent. The largest e-self-consistent set and the largest p-self-consistent set are simply the largest subset of e-BH and BH, respectively, that satisfies the DAG constraints.


\begin{figure}[h]
\center
\includegraphics[width=4in]{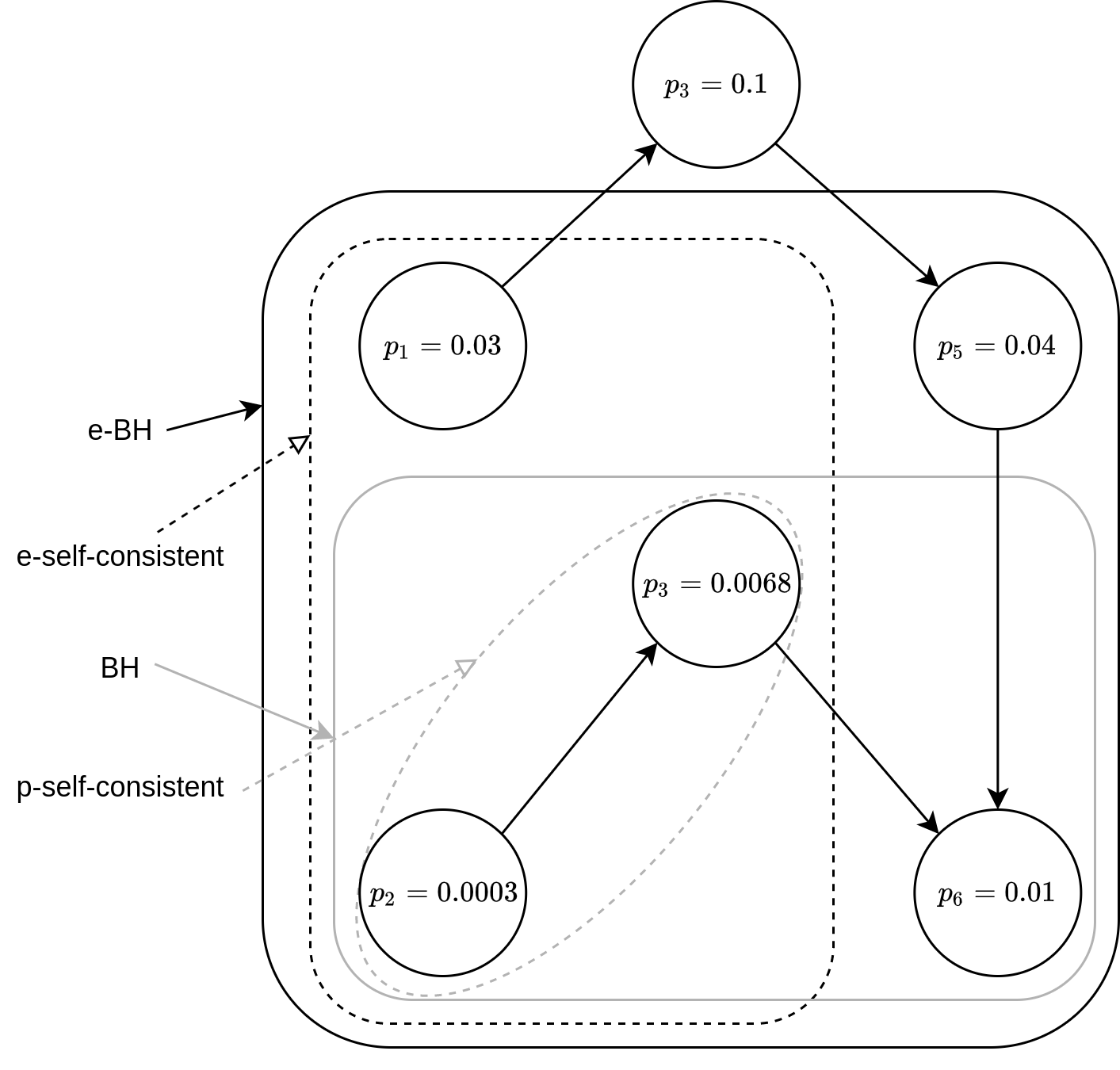}
\caption{Example set of p-values for hypotheses that have a DAG constraint upon them and rejection sets that ensure \(\FDR(\rejset) \leq \delta = 0.05\). We assume the p-variables are arbitrarily dependent, and are reciprocals of e-variables for the sake of comparing e-variable vs. p-variable procedures. The e-self-consistent and p-self-consistent rejection sets are the largest such sets that satisfy the \(\FDR\) guarantee and the DAG constraints. The e-BH and BH rejection sets violate the DAG constraints, i.e.\ they are not valid rejection sets, but they do maintain \(\FDR(\rejset) \leq \delta\). 
The largest valid e-self-consistent rejection set and p-self-consistent rejection set are simply the largest subsets that satisfy the DAG constraints of e-BH and BH, respectively.
}
\label{fig:DAG}
\end{figure}
\vspace{10pt}
\subsection{Multiple agents}
\label{subsec:MultipleAgents}

There are many scenarios where multiple agents are interacting with the same bandit and we hope to have the agents cooperatively accumulate evidence. For example, a research group could be interested in resuming the study of a hypothesis that previous researchers have run experiments on, and would like to combine existing evidence with the new evidence they collect from their own experiments. A cooperative situation could also arise when there are multiple groups that each work on a subset of some overarching set of hypotheses --- the groups can combine the evidence they have for each hypothesis. In these cases, the evidence shared, either from previous studies or concurrent collaborators, might only be in the form of an e-value or p-value --- the actual samples may be obfuscated for privacy reasons. Thus, each of the scenarios require the merging of multiple statistics (from each agent) into a single statistic representing the total amount of evidence for rejecting a hypothesis.

Assume we have \(m\) agents and let \(E_1, \dots, E_m\) denote the e-variables all testing the same hypothesis. If the e-variables are all independent, we can define an ie-merging function (outputs an e-variable from independent e-variables) \(f_{\mathrm{prod}}\) as follows:
\begin{align*}
    f_{\mathrm{prod}}(E_1, \dots, E_m) \coloneqq \prod\limits_{i = 1}^m E_i.
\end{align*}
\begin{proposition}
If \(E_1, \dots, E_m\) are independent e-variables, then \(f_{\mathrm{prod}}(E_1, \dots, E_m)\) is also an e-variable.
\end{proposition} The above proposition follows from the fact that the expectation of the product is the product of expectation for independent random variables.

If \(E_1, \dots E_m\) are dependent, then we can define the following e-merging function (outputs an e-variable from \textit{arbitrarily dependent} e-variables):
\begin{align*}
    f_{\mathrm{mean}}(E_1, \dots, E_m) \coloneqq \frac{1}{m}\sum\limits_{i = 1}^m E_m.
\end{align*}
\begin{proposition}
If \(E_1, \dots, E_m\) are arbitrarily dependent e-variables, then \(f_{\mathrm{mean}}(E_1, \dots, E_m)\) is also an e-variable.
\label{prop:EMergingMean}
\end{proposition}
\citet{vovk_evalues_calibration_2020} show that the set of functions corresponding to all convex combinations of \(f_{\mathrm{mean}}\) and 1 are the only admissible e-merging functions in the class of all symmetric e-merging functions. They also show a weaker sense of dominance for \(f_{\mathrm{prod}}\) --- they prove it outputs a larger e-value than any other symmetric ie-merging function if all the input e-values are at least 1. Thus, e-variables can be merged in a relatively simple fashion without many assumptions.

On the other hand, merging p-variables is difficult. When the p-variables are independent, \citet{birnbaum_combining_independent_1954} show that any valid merging function which is monotonic w.r.t.\ to the p-values is admissible. When the p-variables are arbitrarily dependent, \citet{vovk_admissible_ways_2021} prove that there are also many admissible symmetric p-merging functions. Consequently, the p-variable picture is much less clear about how to optimally merge p-variables, particularly when there is arbitrary dependence among them. To illustrate these e-merging/ie-merging functions can be used in a bandit setting, we consider an example multi-agent problem where many research groups are submitting studies to the same journal.

\subsubsection{Example: controlling the \(\FDR\) of results in a journal}
We consider a situation where the editors of a journal are interested in guaranteeing the accuracy of the results published within the journal. Specifically, they aim to ensure \(\FDR\) control on the discoveries within the papers accepted to the journal. The journal requires that each study that is submitted is also accompanied by an e-value. Since many groups can be testing the same hypothesis, the journal can use the aforementioned merging techniques to combine the e-values reported by different groups and produce a valid, aggregate e-variable. 
\begin{figure}[h]
    \centering
    \vspace{10pt}
    \includegraphics[width=0.7\textwidth]{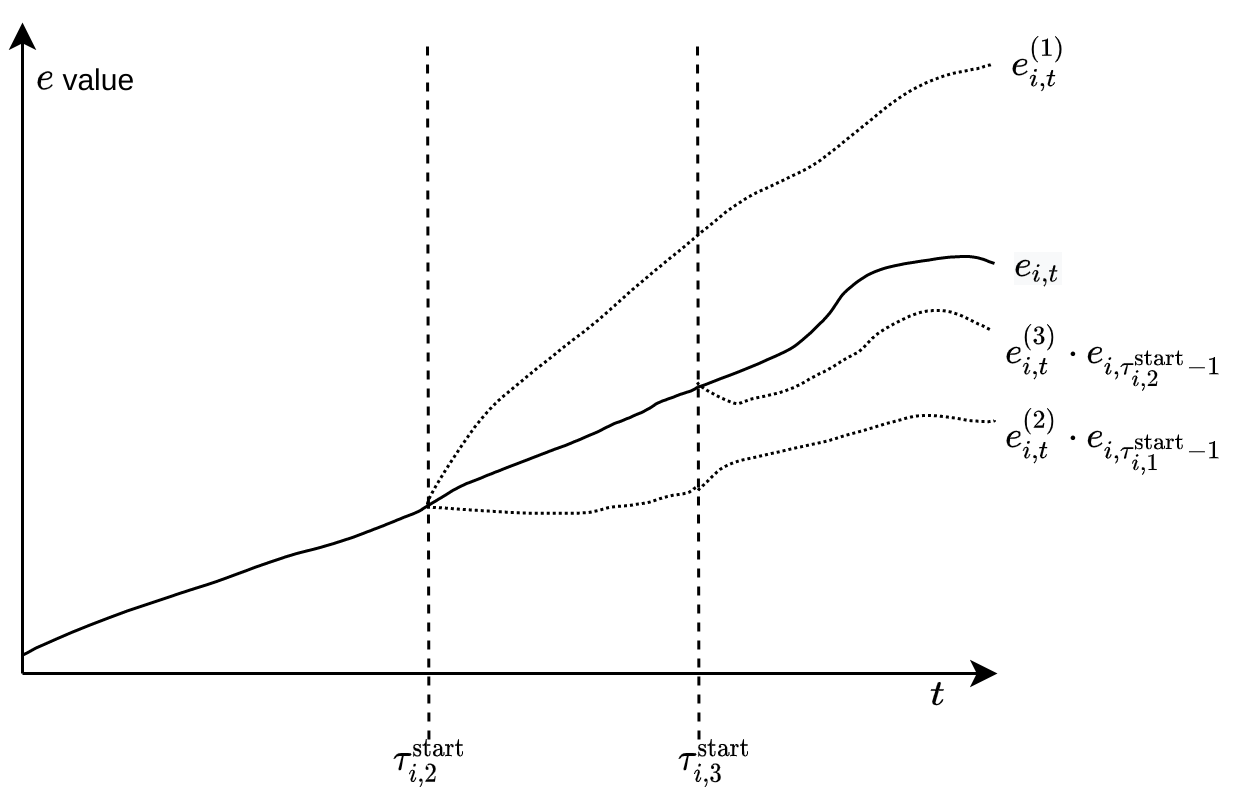}
    \caption{An illustration of how \(e_{i, t}\) changes in relation to each \(e_{i, t}^{(\ell)}\) for a case where \(\ell = 3\) in \Cref{alg:MultiAgent}. We see that \(e_{i, t}\) and \(e_{i, t}^{(1)}\) are identical up to \(\startime_{i, 2} - 1\), where agent 2 begins to sample arm \(i\). Agent 2's process starts at \(e_{i, \startime_{i, 2} - 1}\). Similarly, Agent 3's process starts at \(e_{i, \startime_{i, 3} - 1}\) when it starts to sample \(i\) as well. Validity of \Cref{alg:MultiAgent} arises from the fact that each new agent \(\ell\) has its own \(e_{i, t}^{(\ell)}\) scaled by the \(e_{i, t}\) that has been achieved already.}
    \label{fig:Journal}
    \vspace{10pt}
\end{figure}

\paragraph{Formalizing the multi-agent setup} The reward of the \(i\)th arm on the \(t\)th day for the \(\ell\)th agent is denoted as \(X^{(\ell)}_{i, t} \) for all \(t, \ell \in \naturals\) and \(i \in [k]\). We let the index for agents, \(\ell\), be in \(\naturals\) to allow for arbitrarily large, but finite, number of agents at each time step. We let the joint distribution of \(X_{1, t}^{(\ell)}, \dots,  X_{k, t}^{(\ell)}\) be identically distributed across \(\ell, t \in \naturals\). Consequently, the rewards \(X_{1, t}^{(\ell)}, \dots,  X_{k, t}^{(\ell)}\) corresponding to each agent \(\ell\) are identical in a marginal sense across all \(\ell \in \naturals\). However, there can be arbitrary dependencies between the rewards of different agents. Thus, we allow for a setting where, for each \(i \in [k]\) and  \(t \in \naturals\), \(X_{i, t}^{(\ell)}\) is the same reward across all \(\ell \in \naturals\), and a setting where \(X_{i, t}^{(\ell)}\) are independent across \(\ell \in \naturals\). Each agent \(\ell \in \naturals\) outputs \(\sampleset^{(\ell)}_t \in \Kcal \cup \{\emptyset\}\) for each \(t \in \naturals\). Let the set of agents (e.g.\ set of studies) on day \(t\) that are testing hypothesis \(i\) be \(A_{i, t}\) for each \(t \in \naturals\) and \(i \in [k]\). Critically, we require that \(A_{i, t}\) be of finite cardinality almost surely and predictable w.r.t.\ the new canonical filtration \((\multifiltration_t)\). We define the canonical filtration for the multi-agent setting as follows:
\begin{align*}
\multifiltration_t \coloneqq \sigma(\privrv  \cup \{(i, s, \ell, X_{i, s}^{(\ell)}): i \in \sampleset_s^{(\ell)}, s \leq t, \ell \in A_{i, s}\}).
\end{align*} We denote the e-process of the \(\ell\)th agent for hypothesis \(i\) to be \((E_{i, t}^{(\ell)})\), where \(E_{i, t}^{(\ell)} = 1\) if \(\ell \not\in A_{i, t}\). Implicitly, there exists a stopping time \(\startime_{i, \ell}\) w.r.t. \((\multifiltration_t)\) that denotes the time when the \(\ell\)th agent begins testing the \(i\)th hypothesis for \(\ell \in \naturals\) and \(i \in [k]\) (i.e.\ the time of the first sample of arm \(i\) by agent \(\ell\)). \Cref{alg:MultiAgent} explicitly formulates the algorithm for dealing with e-values coming from multiple agents.
\begin{algorithm}[h]
\caption{An algorithm for aggregating evidence in the form of e-values from many agents. The algorithm takes the mean of the e-values for each hypothesis on each day and applies an e-self-consistent procedure to these aggregated e-values to maintain valid \(\FDR\) control at \(\delta\).}
\KwIn{A level of control \(\delta\) in \((0, 1)\). \((e^{(\ell)}_{i, t})\) are the realized values of an e-process for \(\ell \in \naturals, i \in [k]\).}
\textbf{Initialize} \(e_{i, 0} \coloneqq 1\) for \(i \in [k]\)\\
\(\rejset_0 \coloneqq \emptyset\)\\
\For{\(t \in 1, \dots\)}{
    Receive new results from new or existing agents and update \(e_{i, t}^{(j)}\) for \(i \in [k]\) and \(j \in [a_t]\)\\
    \For{\(i \in [k]\)}{
        \(e_{i, t} \coloneqq 
        \begin{cases}
        \frac{1}{|A_{i, t}|}\sum\limits_{j \in A_{i, t}}e_{i, \startime_{i, j} - 1} \cdot e_{i, t}^{(\ell)} & \text{if }i\not\in \rejset_{t - 1}\\
        e_{i, t - 1} & \text{else}
        \end{cases}\).
    }
    \(\rejset_t \coloneqq \EBH[\delta](e_{1, t}, \dots, e_{k, t})\) or an arbitrary e-self-consistent set.
}
\label{alg:MultiAgent}
\end{algorithm}

\Cref{fig:Journal} illustrates how \(e_{i, t}\) behaves w.r.t.\ to the \(e_{i, t}^{(\ell)}\) of each agent in \Cref{alg:MultiAgent}. \Cref{alg:MultiAgent} uses a merging approach that is in between \(f_{\mathrm{prod}}\) and \(f_{\mathrm{mean}}\). Intuitively, we know the rewards across \(t \in \naturals\) are independent, and consequently we can merge e-values by taking the product. When merging across different \(\ell\in \naturals, i \in [k]\), however, there may be arbitrary dependence between rewards. Consequently, we must take the mean of those e-values. From a betting perspective as discussed in \citet{shafer_testing_betting_2021}, we can view our algorithm as splitting the current wealth (current \(e_{i, t}\)) evenly across each agent whenever a new agent is introduced before allowing each agent to continue or begin its own strategy. Regardless, we can show the following guarantee concerning \Cref{alg:MultiAgent}.
\begin{proposition}
Let \((E_{i, t}^{(\ell)})\) be upper bounded by some nonnegative supermartingale \((M_{i, t}^{(\ell)})\) w.r.t.\ \((\multifiltration_t)\) for \(i \in [k], \ell \in \naturals\) where \(E_{i, t}^{(\ell)} = M_{i, t}^{(\ell)} = 1\) for \(t < \startime_{i, \ell}\). \Cref{alg:MultiAgent} ensures that \(\sup_{\tau \in \Tcal} \FDR(\rejset_\tau) \leq \delta\).
\label{prop:MultiAgentFDR}
\end{proposition}
\begin{proof}
Define \(M_{i, t} \coloneqq \frac{1}{|A_{i, t}|}\sum_{\ell \in A_{i, t}}M_{i, t}^{(\ell)} \cdot M_{i, \startime_{i, \ell} - 1}\) when \(i \not\in \rejset_{t - 1}\) and \(M_{i,t} \coloneqq M_{i, t - 1}\) otherwise. We can see that \(M_{i, t}\) upper bounds \(E_{i, t}\) for all \(t \in \naturals, i \in [k]\). We will show that \((M_{i, t})\) is a nonnegative supermartingale. Assume that we have not rejected the \(i\)th hypothesis yet, since otherwise \(M_{i, t} = M_{i, t - 1}\), which satisfies the supermartingale property.
\begin{align*}
    \expect[M_{i, t} \mid \multifiltration_{t - 1}] &= \expect\left[\frac{1}{|A_{i, t}|} \sum\limits_{\ell \in A_{i, t}} M_{i, t}^{(\ell)} \cdot M_{i, \startime_{i, \ell} - 1} \mid \multifiltration_{t - 1}\right]\\
    &=\frac{1}{|A_{i, t}|}\left(\sum\limits_{\ell \in A_{i, t - 1}} \expect\left[M_{i, t }^{(\ell)} \cdot M_{i, \startime_{i, \ell} - 1}  \mid \multifiltration_{t - 1}\right] + \sum\limits_{\ell \in A_{i, t} \setminus A_{i, t - 1}}\expect\left[M_{i, t}^{(\ell)} \cdot M_{i, t - 1} \mid \multifiltration_{t - 1}\right]\right)\\
    &=\frac{1}{|A_{i, t}|}\left(\sum\limits_{\ell \in A_{i, t - 1}} M_{i, t }^{(\ell)} \cdot M_{i, \startime_{i, \ell} - 1}  + \sum\limits_{\ell \in A_{i, t} \setminus A_{i, t - 1}}\expect\left[M_{i, t}^{(\ell)} \mid \multifiltration_{t - 1}\right]  \cdot M_{i, t - 1} \right)\\
    &\leq\frac{1}{|A_{i, t}|}\left(\sum\limits_{\ell \in A_{i, t - 1}} M_{i, t - 1}^{(\ell)} \cdot M_{i, \startime_{i, \ell} - 1}  + \sum\limits_{\ell \in A_{i, t} \setminus A_{i, t - 1}} M_{i, t - 1} \right)\\
    &=\frac{1}{|A_{i, t}|}\left(|A_{i, t - 1}| \cdot M_{i, t - 1}  + |A_{i, t} \setminus A_{i, t -1}| \cdot  M_{i, t - 1} \right)\\
    &=M_{i, t- 1}.
\end{align*}
The sole inequality is because \((M_{i, t}^{(\ell)})\) is a supermartingale, and \(M_{i, t}^{(\ell)} = 1\) when \(t < \startime_{i, \ell}\). Thus, \(\sup_{\tau \in \Tcal} \expect[E_{i, \tau}] \leq \sup_{\tau \in \Tcal} \expect[M_{i, \tau}] \leq 1\) where the final inequality is by optional stopping. Consequently, \((E_{i, t})\) are e-processes for \(i \in [k]\) so \(\sup_{\tau \in \Tcal} \FDR(\rejset_\tau) \leq \delta\) by \Cref{fact:EComplianceFDR}, which achieves our desired result.
\end{proof}

Nonnegative martingales play a central role in characterizing admissible e-processes --- every e-process is upper bounded by a nonnegative martingale (Corollary 24; \citet{ramdas_admissible_2020}). Thus, \Cref{prop:MultiAgentFDR} proves that if \((E_{i, t}^{(\ell)})\) are all e-processes for \(i \in [k], \ell \in \naturals\), then \(\FDR\) control is maintained in the multi-agent for any stopping time.

\subsection{Hypotheses involving multiple arms}
\label{subsec:MultipleArms}

In the current setting, we have only considered hypotheses that are tied to a single arm i.e.\ hypothesis \(i\) is concerned solely with \(\nu_i\) for all \(i \in [k]\). We also might be concerned with hypotheses that involve multiple arms. For example, we could be interested in the hypothesis that the reward distributions are exchangeable across arms \citep{vovk_testing_randomness_2021,vovk_retrain_not_2021a} i.e.\ any permutation of the arms is the same distribution, or the hypothesis that the means of two specific reward distributions are the same. Naturally, if each hypothesis is not restricted to being involved with only a single arm, we can consider more (or fewer) hypotheses than the number of arms.

Thus, we can denote \(k\) to be the total number of hypotheses and \(n\) to be the number of arms. \Cref{alg:MultiArmHypotheses} specifies a meta-algorithm similar to \Cref{alg:Framework} that maintains \(\FDR\) control in multi arm hypotheses. We simply maintain an e-process or p-process for each hypothesis. An important difference between hypotheses involving multiple arms setting and the standard setting is that the independence of \(X_{1, t}, \dots, X_{n, t}\) is no longer sufficient to ensure all the e-variables or p-variables are dependent only through the exploration policy and stopping time. The dependence structure within the e-variables or p-variables is based not only upon the dependence of \(X_{1, t}, \dots, X_{n, t}\), but also whether the hypothesis tests themselves have any dependence among each other e.g. two hypotheses might involve the same arm. Thus, for p-variables, we may require \(\delta' = \delta / \log k\) even when the reward distributions are independent.

\begin{algorithm}
\caption{A meta-algorithm that ensures \(\FDR\) control when hypotheses can involve multiple arms in the bandit setting.}
\KwIn{Exploration component \((\EC_t)\), stopping rule \(\stoptime\),
desired level of \(\FDR\) control \(\delta \in (0, 1)\). Set $D_0 = \emptyset$.}
\For{\(t\) in \(1 \dots\)}{
    \(\sampleset_t \coloneqq \EC_t(D_{t-1}) \subseteq [n]\)\\
    Obtain rewards for each $i \in \sampleset_t$, and update data $D_t:=D_{t-1} \cup \{(i, t, X_{i, t}): i \in \sampleset_t\}$.\\
    Update e-process or p-process that relate to any of the queried arms.\\
    \(\rejset_t \coloneqq 
    \begin{cases}
    \BH[\delta / \log k](p_{1, t}, \dots, p_{k, t}) \text{ or arbitrary p-self-consistent set} & \text{if using p-variables}\\
    \EBH[\delta](e_{1, t}, \dots, e_{k, t}) \text{ or arbitrary e-self-consistent set}& \text{if using e-variables}
    \end{cases}\)\\
    \lIf{\(\stoptime = t\)}{stop and \Return \(\rejset_t\)}
}
\label{alg:MultiArmHypotheses}
\end{algorithm}
\begin{proposition}
\Cref{alg:MultiArmHypotheses} outputs \(\rejset_t\) for all \(t \in \naturals\) such that \(\sup_{\tau \in \Tcal} \FDR(\rejset_\tau) \leq \delta\).
\end{proposition}

E-variables in this setting have potentially larger power over p-variables than in the standard setting. This is because the number of hypotheses, \(k\), is no longer tied to the number of arms, \(n\). For example, \(k \approx n^2 / 2\) if there was a hypothesis for each pair of arms in the bandit. Then, using p-variables in \Cref{alg:MultiArmHypotheses} would require a correction of approximately \(2 \log k\). In contrast, p-variables and BH require no more than a \(\log k\) correction in the standard setting. Consequently, allowing for multiple arm hypotheses further highlights the benefit of e-variables over p-variables when dealing with arbitrarily dependent statistics.

\section{Additional simulations}
\label{sec:AdditionalSimulations}
In this section, we perform additional simulations to empirically verify our theoretical results. We test the performance of different choices of p-variables against e-variables in the standard bandit setting. We also provide simulations for the combinatorial bandit setting and compare p-variable methods with different assumptions against an e-variable method.

\subsection{Testing against different choices of p-variables}
\label{subsec:PVariableSimulations}

We consider two additional choices of p-variables to compare with our e-variable method and the p-variable from \JJ\ discussed in \Cref{sec:Experiments}. One is simply \(P_{i, t}^{\ipmh} \coloneqq 1 / \pmh_{i, t}\), which we will call Inverse PM-H (IPM-H). The other, which we call the IS p-variable, which is defined as follows by setting \(\varphi = \isphi\) in \eqref{eqn:PVariable}.
\begin{align}
    \isp_{i,t} &\coloneqq \inf\{\beta \in [0, 1]: |\widehat{\mu}_{i, t} - \mu_0| > \isphi(t, \beta)\}.
\end{align}

We run these methods using the UCB arm selection algorithm described in \eqref{eqn:UCB} inside of \Cref{alg:Framework}.

\begin{figure}[h]
    \captionsetup[subfigure]{aboveskip=-5pt}
    \centering
    \begin{subfigure}[b]{0.32\textwidth}
    \adjustbox{trim=0 0 0 {.2\height},clip}{\includegraphics[width=0.9\textwidth]{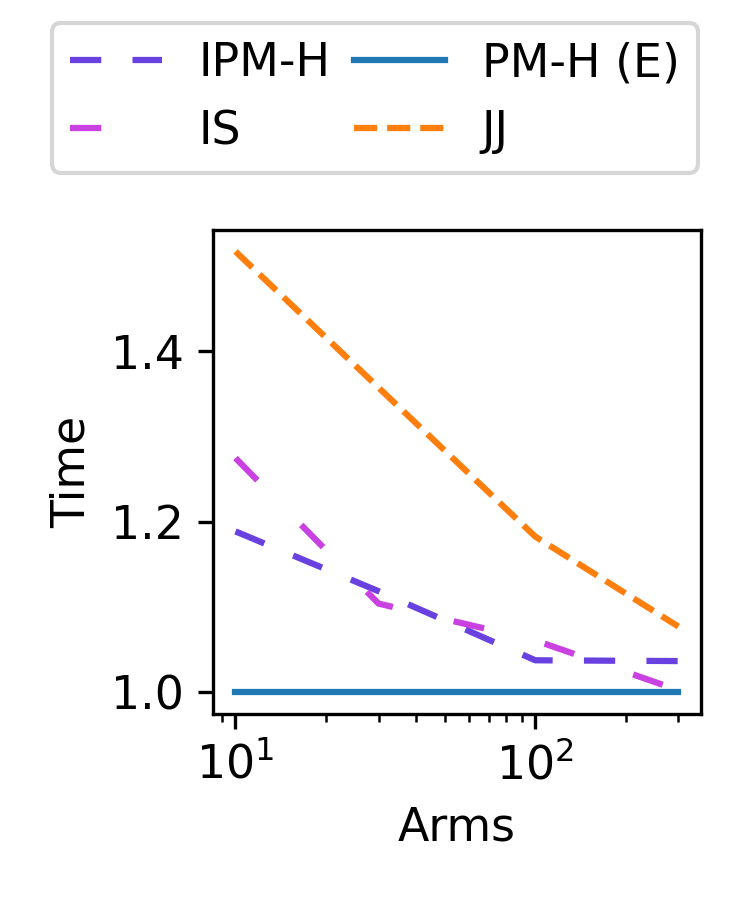}}
    \caption{\(|\hypset_1| = 2\)}
    \end{subfigure}
    \begin{subfigure}[b]{0.32\textwidth}
    \includegraphics[width=0.9\textwidth]{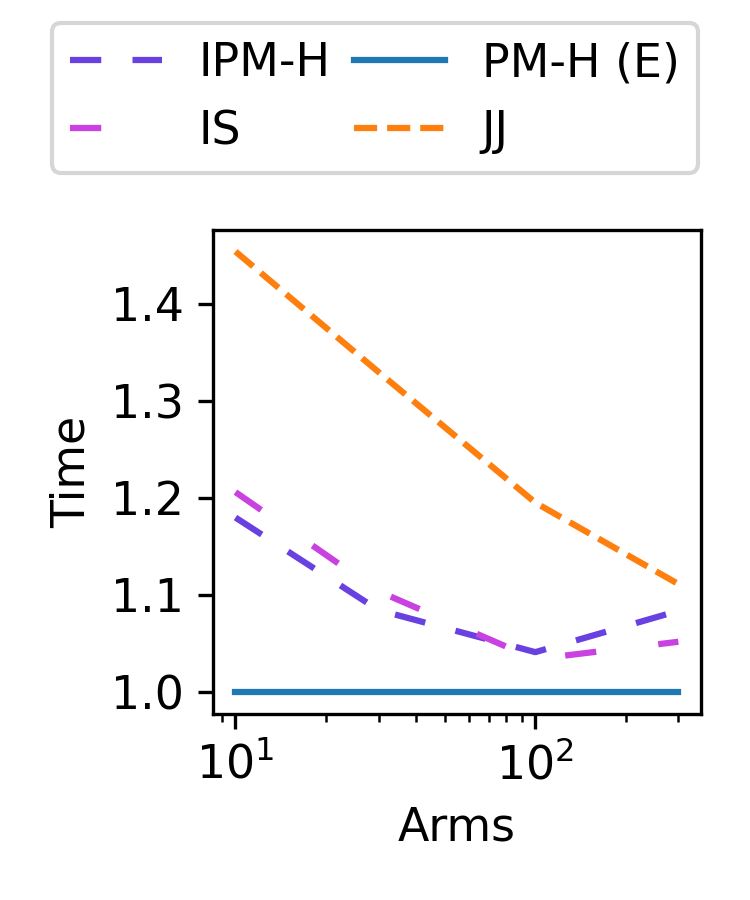}
    \caption{\(|\hypset_1| = \lfloor \log k \rfloor\)}
    \end{subfigure}
    \begin{subfigure}[b]{0.32\textwidth}
    \adjustbox{trim=0 0 0 {.2\height},clip}{\includegraphics[width=0.9\textwidth]{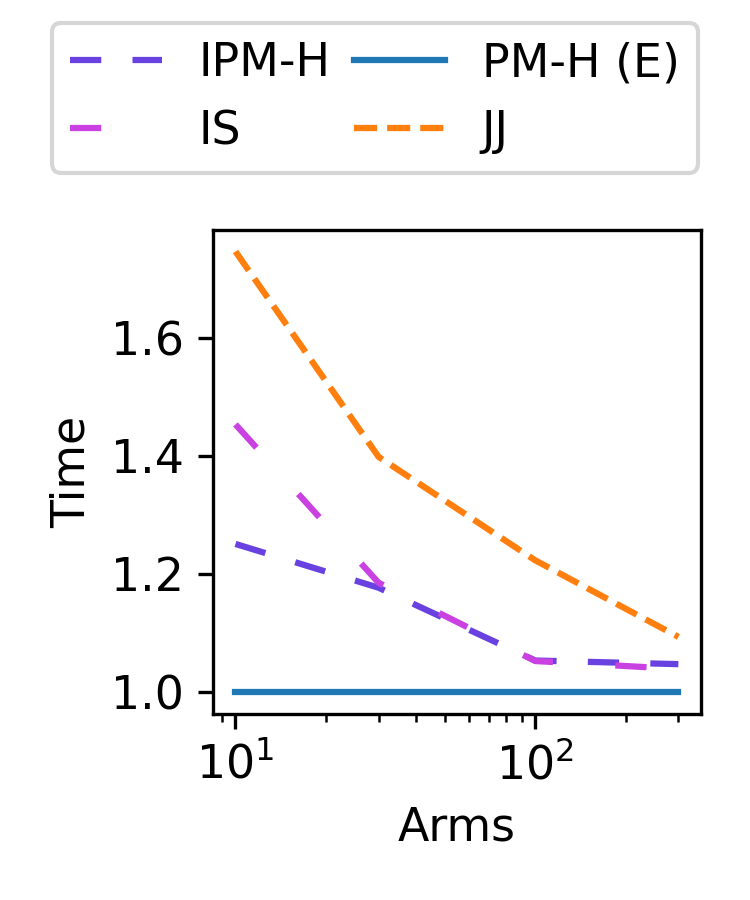}}
    \caption{\(|\hypset_1| = \lfloor \sqrt{k} \rfloor\)}
    \end{subfigure}
    \caption{Relative comparison of time \(t\) for each method to obtain a rejection set, \(\rejset_t\), that has a \(\TPR(\rejset_t) \geq 1 - \delta\) while maintaining \(\FDR(\rejset_t) \leq \delta\), where we choose \(\delta = 0.05\). This plot compares different choices of p-variables against the PM-H e-variable over different numbers of arms (choices of \(k\)) and different densities of non-null hypotheses (sizes of \(\hypset_1\)). Time is reported as a ratio to the time taken by the algorithm that uses the PM-H e-variable. The \JJ\ p-variable is the baseline p-variable specified in \eqref{eqn:PAlgorithm}. We see that both the IPM-H and the IS p-variable have similar performance, and require fewer samples than the \JJ\ p-variable. Overall, the PM-H e-variable performs better than any choice of p-variable.}
    \label{fig:PVars}
\end{figure}

The results shown in \Cref{fig:PVars} demonstrate that e-variables and e-BH still perform better than any p-variable and BH method. The two new p-variables, IS and IPM-H, have about similar sample efficiency, and both outperform the \JJ\ p-variable, but both are still slightly worse than the PM-H e-variable. Thus, e-BH and e-variables have consistently better performance than BH and p-variables.

\subsection{Graph bandits with dependent \(X_{1, t}, \dots, X_{k, t}\)}
\label{subsec:GraphSimulations}
We consider a graph bandit setting where the algorithm makes no assumptions about the underlying dependence structure, and each arm consists of a node and its neighbors. We set the joint distribution over rewards at each step as the product of independent normal distributions for each arm. The marginal distribution of each arm \(i \in [k]\) is a normal distribution with mean \(\mu_i\), where \(\mu_i = 1/2\) if \(i \in \hypset_1\) and \(\mu_i = \mu_0 = 0\) if \(i \in \hypset_0\). Each graph we simulate is composed of 10 cliques of \(k / 10\) nodes. Thus, the set of superarms available for sampling is \(\Kcal = \{\{i, i + 10, i+ 20, \dots, i + k - 10\}: \text{for }i \in [10]\}\). Finally, we let \(\delta = 0.05\) be level of \(\FDR\) control for each algorithm.

We compare 3 different methods. For all 3 methods, the exploration strategy is to uniformly sample from the set of superarms \(\Kcal\). These methods differ solely in their choice of the evidence component. The first method is called the \textit{single arm BH} method, as it only saves a single uniformly random sample from the set of samples it attains at each time step. Hence, it is equivalent to the uniformly randomly sampling BH method for the standard bandit setting. In this combinatorial bandit setting, it simply discards all but one sample at each step, and can consequently still enjoy the guarantees in \Cref{prop:AdaptiveIndPFDR}. Our second method is to use the default BH and p-variables with no discarding of samples and the larger correction from \Cref{prop:AdaptiveDepPFDR}. Lastly, we have the e-BH and e-variable method that also uses all samples from each pull of a superarm, since e-BH requires no correction for arbitrary dependence.

\begin{figure}[h]
    \captionsetup[subfigure]{aboveskip=-5pt}
    \centering
    \begin{subfigure}[b]{0.32\textwidth}
    \adjustbox{trim={0.083\width} 0  {0.083\width} {.2\height},clip}{\includegraphics[width=0.9\textwidth]{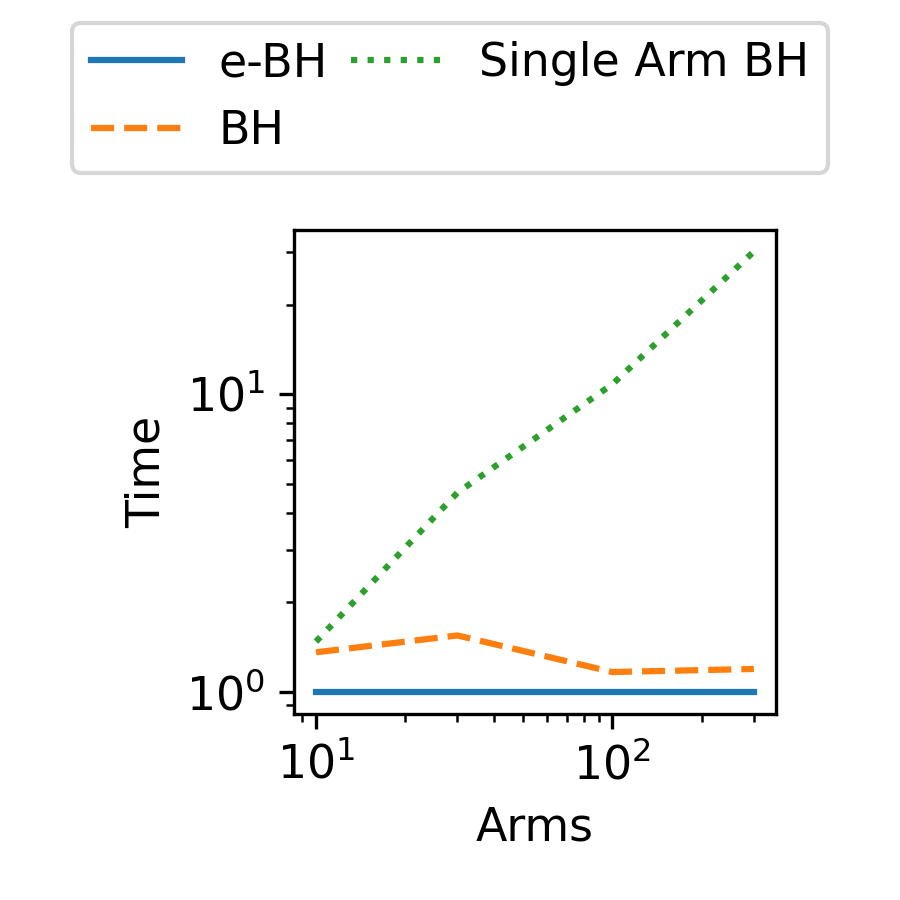}}
    \caption{\(|\hypset_1| = 2\)}
    \end{subfigure}
    \begin{subfigure}[b]{0.32\textwidth}
     \adjustbox{trim={0.05\width} 0  {0.05\width} 0,clip}{\includegraphics[width=0.9\textwidth]{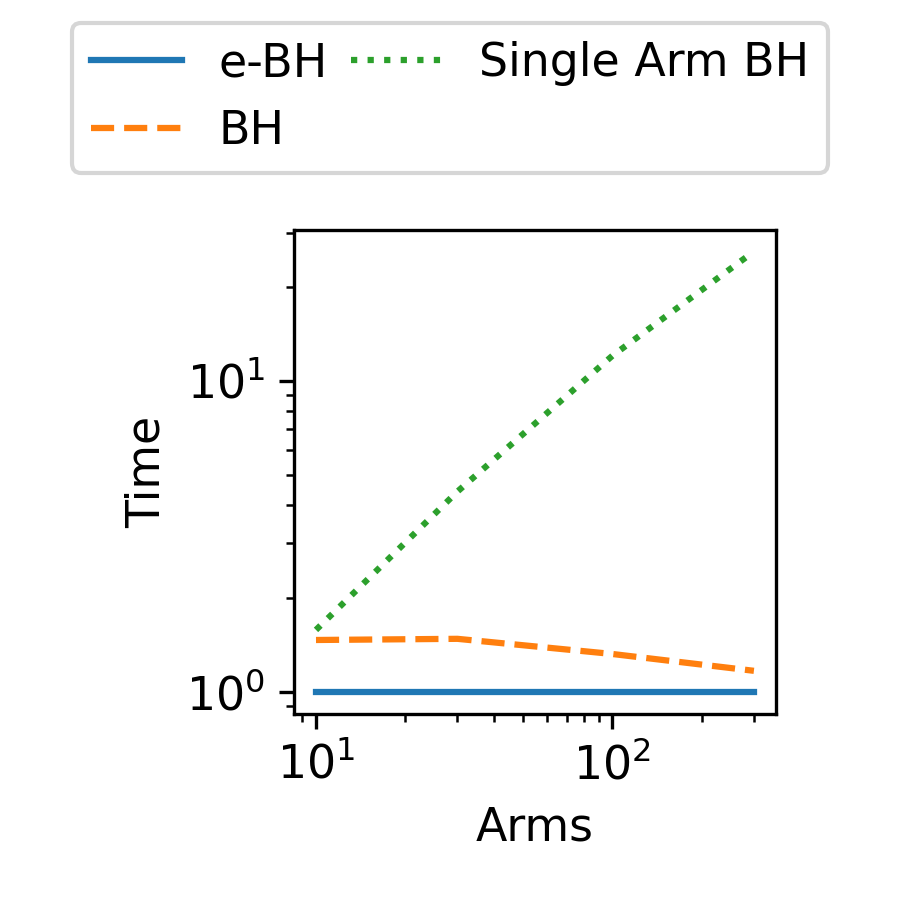}}
    \caption{\(|\hypset_1| = \lfloor \log k \rfloor\)}
    \end{subfigure}
    \begin{subfigure}[b]{0.32\textwidth}
    \adjustbox{trim={0.083\width} 0 {0.083\width} {.2\height},clip}{\includegraphics[width=0.9\textwidth]{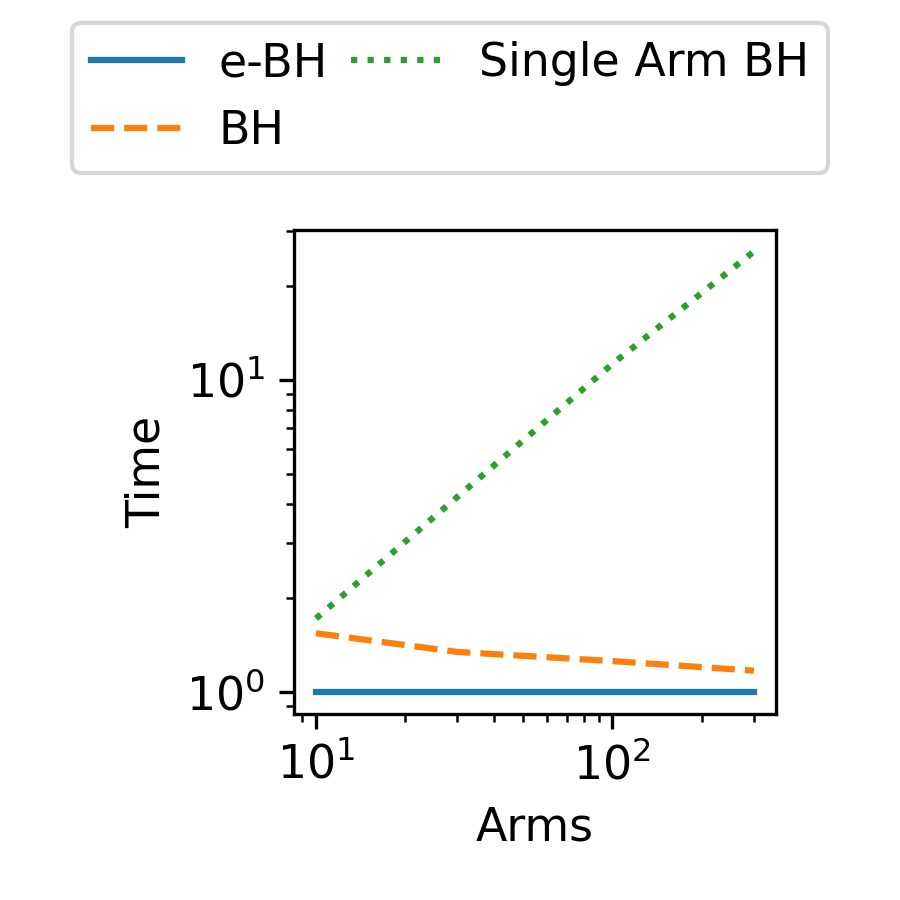}}
    \caption{\(|\hypset_1| = \lfloor \sqrt{k} \rfloor\)}
    \end{subfigure}
    \caption{Relative comparison of time \(t\) for each method to obtain a rejection set, \(\rejset_t\), that has a \(\TPR(\rejset_t) \geq 1 - \delta\) while maintaining \(\FDR(\rejset_t) \leq \delta\), where we choose \(\delta = 0.05\). This plot compares two different p-variable methods (BH and single arm BH) against an e-variable method (e-BH) over different numbers of arms (choices of \(k\)) and different densities of non-null hypotheses (sizes of \(\hypset_1\)). Time is reported as a ratio to the time taken by the algorithm that uses e-BH. We see that e-BH outperforms the two other BH algorithms in the graph bandit setting. Notably, single arm BH is linearly increasing in time relative to the other two methods that make full use of the samples obtained from a superarm. Single arm BH discards too many samples at each step, and the smaller correction it makes does not make up the deficit in number of samples.}
    \label{fig:GraphBanditResults}
    \vspace{10pt}
\end{figure}
\Cref{fig:GraphBanditResults} shows the results of using methods that guarantee \(\FDR\) control at level \(\delta\) on graph bandits with arbitrary dependence between arms. Single arm BH pays a tremendous cost in time by throwing away many samples at each step, and the slightly smaller correction it needs to make does not make up for this deficit. Between the two methods that make full use of the samples obtained from superarm, we see that e-BH does better. Thus, e-variables and e-BH exhibit empirical performance on par or better than p-variables and BH in both the standard and combinatorial bandit settings.

\section{Betting interpretation of e-variables for bandits}
\label{sec:Betting}

We will describe our methodology for constructing e-variables using the perspective of betting in this section.\begin{revision}
\citet{shafer_testing_betting_2021} uses betting to formulate a paradigm for understanding the quantity represented by an e-value, and \citet{shafer_game_theoretic_2019} extend these ideas to form a mathematically rigorous foundation for probability based on game theory. Separately, betting ideas have also been used in parameter free techniques for online learning \citep{orabona_coin_betting_2016a,jun_online_learning_2017a,orabona_training_deep_2017a}. In this section, we will use a betting approach to produce a data adaptive e-process. 
\end{revision}

Recall that if \(E\) is an e-variable, then \(\expect[E] \leq 1\) when the null hypothesis is true. On the other hand, if the null hypothesis is false, we would like \(E\) to be large, since that increases the likelihood that the null hypothesis is rejected. Thus, constructing \(E\) such that is satisfies the e-variable constraint under the null and is large under the alternative is the same as constructing a valid hypothesis test that has as much power as possible. Consequently, we can consider a betting game where we pay a dollar to play, and \(E\) is the payout. If the null hypothesis is true, then we are unable to make any money in expectation, since the expectation is of \(E\) is at most 1. However, if the null hypothesis is false, then we would expect to be able to make money on this game. If we did not make money under the alternative, then any test that used this e-variable would have no power, since the behavior of \(E\) would not change between the null hypothesis being true and being false. In other words, this would be no better than picking \(E = 1\) deterministically: a valid e-variable, but ineffectual for testing.

We define the \textbf{predictably-mixed Hoeffding (PM-H)} e-process \citep{waudby-smith_estimating_means_2021}, which we used in our simulations in \Cref{sec:Experiments}, as follows:
\begin{align*}
    \pmh_{i, t}(\mu_0) \coloneqq \prod\limits_{j = 1}^{T_i(t)}\exp(\lambda_{i, t_i(j)}(X_{i, t_i(j)} - \mu_0) - \lambda_{i, t_i(j)}^2 / 2).
\end{align*}\((\lambda_{i, t})\) is any sequence of nonnegative real numbers that is predictable w.r.t.\ \((\filtration_t)\). We will use an argument based on betting to derive a \((\lambda_{i, t})\) sequence, and show that this e-process can ``make money'' and hence provide \(\TPR\) guarantees in the sub-Gaussian case. We first observe the following property of this process.

\begin{proposition}
\(\pmh_{i, t}(\mu_0)\) is a nonnegative supermartingale, and thus an e-process, if $i \in \hypset_0$ and \(\nu_i\) is \(1\)-sub-Gaussian.
\label{prop:SafeBettingVariable}
\end{proposition}
\begin{proof}
We drop \(\mu_0\) from \((\pmh_{i, t}(\mu_0))\) and denote it as \((\pmh_{i, t})\).

We proceed by showing \((\pmh_{i, t})\) is a nonnegative supermartingale w.r.t.\ to the canonical filtration \((\filtration_t)\). 

Consider \(\pmh_{i, t}\) when \(i \in \hypset_0\). If \(I_t \neq i\) then \(\pmh_{i, t} = \pmh_{i, t - 1}\), which satisfies the supermartingale property in \eqref{eqn:Supermartingale}.

Otherwise,
\begin{align*}
    \expect[\pmh_i{i, t} \mid \filtration_{t - 1}]& = \expect\left[\exp\left(\lambda_t (X_{i, t} - \mu_0) - \frac{\lambda_t^2}{2}\right)\pmh_i{i, t - 1} \mid \filtration_{t - 1}\right]\\
    &=\expect\left[\exp\left(\lambda_t (X_{i, t} - \mu_0) - \frac{\lambda_t^2}{2}\right) \mid \filtration_{t - 1}\right]\pmh_{i, t - 1} \\
    & \leq \pmh_{i, t - 1},
\end{align*} where the final equality is because \(X_{i, t}\) are independent across \(t \in \naturals\) and \(1\)-sub-Gaussian,  and \(\mu_i \leq \mu_0\).

Since \((\pmh_{i, t})\) is a nonnegative supermartingale, it is an e-process by optional stopping. Thus, we have achieved our desired result.
\end{proof}

Note that \Cref{prop:SafeBettingVariable} justifies that our choice of e-process for the simulations in \Cref{sec:Experiments} was indeed a valid e-process. Now that we have shown \((\pmh_{i, t})\) is an e-process, we will consider how to choose a powerful \((\lambda_{i, t})\). Consider a model where we view the e-value, \(e_{i, t}\), for each arm \(i \in [k]\) as the money made by each arm, or a ``\textbf{betting score}''. For each arm \(i \in [k]\), imagine we are allocated initial wealth equal to 1. At each time step, the algorithm chooses an arm \(i \in [k]\), and a ``bet'', \(\lambda_{i, t}\). The wealth of the arm at the next round changes by a factor based on the reward \(X_{i, t}\) (assuming \(i\) is the arm chosen at round \(t + 1\)): 
\begin{align*}
e_{i, t + 1} = e_{i, t} \cdot \underset{\text{change in wealth}}{\underbrace{\exp(\lambda_{i, t}(X_{i, t} - \mu_0) - \lambda_{i, t}^2 / 2)}}.
\end{align*} Note that this a ``fair game'' or the reward multiplier is less than 1 in expectation if \(\expect[X_{i, t}] \leq \mu_0\). 

The betting score, \(\pmh_{i, t}\), may be interpreted as the money earned by arm \(i\) at time \(t\). When the null hypothesis is true, i.e.\  \(\mu_i \leq \mu_0\), we know that \(\sup_{\tau \in \Tcal} \expect[\pmh_{i, \tau}] \leq 1\) by \Cref{prop:SafeBettingVariable}. Thus, regardless of our stopping strategy, we make no money in expectation. However, if we knew that \(\pmh_{i, t}\) was actually a favorable bet, and \(\expect[X_{i, t}] = \mu_i > \mu_0\), we would want to come up with a sequence \((\lambda_{i, t})\) for each arm \(i \in [k]\) that maximizes our wealth at each arm. Consequently, we can reframe our goal for choosing \((\lambda_{i, t})\) as maximizing capital in a betting game. In the next section, we will discuss some strategies for accomplishing such an objective.

\subsection{Optimal betting strategies}    
One way of maximizing capital is to optimize for the Kelly criterion \citep{kelly_new_1956}, which aims to maximize the logarithm of the capital on each step and is equivalent to maximizing rate of growth of capital. In our scenario, the Kelly criterion manifests in the following form:

\begin{align*}
    \expect\left[\log \pmh_{i, t}(\mu_0)\right] = \sum\limits_{j = 1}^{T_i(t)}\expect\left[\lambda_{i, t_i(j)}(X_{i, t_i(j)} -\mu_0) - \lambda_{i, t_i(j)}^2 / 2\right].
\end{align*}

\paragraph{Optimal choice of \((\lambda_t)\) for log wealth.} To maximize the above sum, we can simply decompose it with respect to each \(j\), and since the \(\lambda_{i, t_i(j)}\) are decoupled, we can identify an optimal \(\lambda_{i, t_i(j)}^*\) for each \(j\):
\begin{gather*}
    \lambda_{i, t_i(j)}^*\coloneqq \argmax_{\lambda \in \reals^+}\  \lambda \expect[X_{i, t_i(j)} - \mu_0] - \lambda^2 / 2 = \mu_i,\\
    \lambda_{i, t_i(j)}^* \expect[X_{i, t_i(j)} - \mu_0] - {\lambda_{i, t_i(j)}^*}^2 / 2 = \max_{\lambda \in \reals^+}\  \lambda \expect[X_{i, t_i(j)} - \mu_0] - \lambda^2 / 2 = \gap_i^2 / 2.
\end{gather*}

We can see that if the \(\mu_i\) is known, the above quantity is maximized by setting \(\lambda_{i, t_i(j)} = \mu_i\)  for all \(j \in [T_i(t)]\). This observation confirms our intuition that the Kelly criterion is a sensible quantity to optimize for when trying to maximize the e-values of hypothesis in \(\hypset_1\). On the other hand, if \(i \in \hypset_0\), \(\mu_i \leq \mu_0 = 0\), the log wealth incurred at each time step is nonpositive. Thus, in expectation, the log wealth process \(\log \pmh_{i, t}(\mu_0)\) will only increase in capital when the hypothesis associated with the arm is truly non-null. In betting language, we are presenting a one-sided bet that allows for our bets \((\lambda_{i, t})\) to make money in expectation iff the null hypothesis is false. Hence, our strategy is profitable only when the true mean of the arm is greater than \(\mu_0\).

In practice, we do not know \(\mu_i\), since testing \(\mu_i\) is the entire premise of the problem. Instead, we can use the sample mean, \(\widehat{\mu}_{i, t}\), in place of \(\mu_i\) and show that it gives us convergence at a rate of approximately \(1 / T_i(t)\) to the optimal capital gain rate.
\begin{proposition}
Let \(\nu_i\) be \(1\)-sub-Gaussian for \(i \in [k]\). If \(\lambda_{i, t} = \hat{\mu}_{i, t - 1}\), then 
\begin{align*}
\expect[\widehat{\mu}_{i, t_i(j) - 1} (X_{i, t_i(j)} - \mu_0) - \widehat{\mu}_{i, t_i(j) - 1}^2 / 2] = \gap_i^2 / 2 - 1 / (T_i(t) - 1).
\end{align*}
\label{prop:ConvergentWealthGrowth}
\end{proposition} \Cref{prop:ConvergentWealthGrowth} follows from the variance of \(\widehat{\mu}_{i, t}\) being \(1 / T_i(t)\). Now, we can derive the following corollary.
\begin{corollary}
The total log wealth at time \(t\), \(\log \pmh_{i, t}\), has an expectation satisfying the following property, where \(\lambda_{i, t} = \widehat{\mu}_{i, t - 1}\):
\begin{align*}
\expect[\log \pmh_{i, t}] &= T_i(t)\gap_i^2 / 2 - \sum\limits_{j = 1}^{T_i(t)} 1 / j \approx T_i(t)\gap_i^2 / 2 - \log(T_i(t)),
\end{align*} where \(\sum\limits_{j = 1}^{T_i(t)} 1 / j\) is approximately \(\log(T_i(t))\).
\end{corollary} Thus, in log wealth, using \(\widehat{\mu}_{i, t}\) incurs a penalty of \(\log T_i(t)\), which is relatively small compared to the positive term --- especially when \(t\) is large.

\subsection{Sample complexity for standard sub-Gaussian bandits} We prove a sample complexity result for the \(\pmh_{i, t}\) as well.
\begin{theorem}
Let \((\EC_t)\) be such that \(\EC_t\) outputs \(\sampleset_t = \{I_t\}\) for all \(t \in \naturals\), where \(I_t\) is defined in \eqref{eqn:UCB}, \(\lambda_{i, t} = (\widehat{\mu}_{i, t - 1} / 2)_+\), and \(E_{i, t} = \pmh_{i, t}\). Then, \Cref{alg:Framework} will always guarantee \(\sup_{\tau \in \Tcal}\FDR(\rejset_\tau) \leq \delta\). With at least \(1 - \delta\) probability, there will exist
\begin{align*}
T \lesssim &\sum\limits_{i \in \hypset_0} \gap_i^{-2}\log(\log(\gap_i^{-2}) / \delta) + \sum\limits_{i \in \hypset_1} \gap_i^{-2}(\log(\gap_i^{-2})\log(1 / \delta) + \log k)\\
&\wedge |\hypset_0|\gap^{-2}\log(\log(\gap^{-2}) / \delta) + |\hypset_1| \gap^{-2}\log(\gap^{-2})\log(1 / \delta) \end{align*}
such that 
\(\TPR(\rejset_t) \geq 1 - \delta\) for all \(t \geq T\).
\label{thm:BettingSampleComplexity}
\end{theorem}

\Cref{thm:BettingSampleComplexity} shows the limitation of using an estimate of the mean, \(\widehat{\mu}_{i, t}\), in place of the true mean. The intuition of the proof of \Cref{thm:BettingSampleComplexity} is that at each step, \(\pmh_{i, t}\) must account for an \(1/ t\) deviation, since the variance of \(\widehat{\mu}_{i, t}\) is \(1 / t\). The sum of these deviations is approximately \(\log t\). Thus, the sample complexity bound has a \(\log \gap^{-2}\) instead of only a \(\log \log \gap^{-2}\) term. This limitation seems to be an inherent flaw in choice of \((\lambda_{i, t})\) based on estimation, since the estimation error must be accounted for along with the typical deviation from providing a concentration inequality that is uniform over time steps \(t\).

To prepare for our proof of \Cref{thm:BettingSampleComplexity}, we require some self-contained lemmata.  Define the following auxiliary random variables for all \(i \in \hypset_1\):
\begin{align}
    \rho_i' \coloneqq \min_{t \in \naturals} \frac{E_{i, t}}{\exp\left(\sum\limits_{j =1}^{T_i(t)}\lambda_{i, t_i(j)}\gap_i - \lambda_{i, t_i(j)}^2\right)}.
\end{align}
\begin{lemma}
For all \(i \in \hypset_1\), \(\prob{\rho_i' \leq s} \leq s\) for \(s \in (0, 1)\) i.e.\ \(\rho_i'\) is superuniformly distributed.
\end{lemma}
\begin{proof} We observe that the reciprocal of \(\rho_i'\) is the following:
\begin{align*}
    1 / \rho_i' = \max_{t \in \naturals}\ \exp\left(\sum\limits_{j = 1}^{T_i(t)} \lambda_{i, t_i(j)}(\mu_i  - X_{i, t_i(j)}) -  \lambda_{i, t_i(j)}^2 / 2\right).
\end{align*} Let
\begin{align*}
    M_t = \exp\left(\sum\limits_{j =1}^{T_i(t)} \lambda_{i, t_i(j)}(\mu_i  - X_{i, t_i(j)}) -  \lambda_{i, t_i(j)}^2 / 2\right).
\end{align*} 

We will show \((M_t)\) is a nonnegative supermartingale w.r.t.\ \((\filtration_t)\). Assume arm \(i\) is sampled at time \(t\) --- otherwise the supermartingale property is trivially satisfied.
\begin{align*}
    \expect[M_t \mid \filtration_{t - 1}] &= \expect\left[\exp\left(\sum\limits_{j =1}^{T_i(t)} \lambda_{i, t_i(j)}(\mu_i  - X_{i, t_i(j)}) -  \lambda_{i, t_i(j)}^2 / 2\right) \mid \filtration_{t - 1}\right]\\
    &=\expect\left[\exp\left(\lambda_{i, t}(\mu_i  - X_{i, t}) -  \lambda_{i, t}^2 / 2\right) \mid \filtration_{t - 1}\right]\exp\left(\sum\limits_{j =1}^{T_i(t - 1)} \lambda_{i, t_i(j)}(\mu_i  - X_{i, t_i(j)}) -  \lambda_{i, t_i(j)}^2 / 2\right)\\
    & \leq \exp\left(\sum\limits_{j =1}^{T_i(t - 1)} \lambda_{i, t_i(j)}(\mu_i  - X_{i, t_i(j)}) -  \lambda_{i, t_i(j)}^2 / 2\right)\\
    & = M_{t - 1}.
\end{align*} The sole inequality arises from \(X_{i, t}\) being independent across \(t \in \naturals\) and \(1\)-sub-Gaussian, and having mean \(\mu_i\).

Thus, \(\rho_i'\) is superuniformly distributed by Ville's inequality.
\end{proof}

Rewriting the definition of \(\rho_i'\), we get:
\begin{align}
    E_{i, t} \geq \exp\left(\sum\limits_{j = 1}^{T_i(t)}\lambda_{i, t_i(j)}\gap_i - \lambda_{i, t_i(j)}^2\right)\rho_i'
    \label{eqn:BettingGrowth}
\end{align} for all \(t \in \naturals\). Now show a result for the rate of growth of \(E_{i, t}\) by showing a result concerning the lower bound in \eqref{eqn:BettingGrowth}.
\begin{lemma}
For all \(t \in \naturals\),
\begin{align*}
\exp\left(\sum\limits_{j = 1}^{T_i(t)}\lambda_{i, t_i(j)}\gap_i - \lambda_{i, t_i(j)}^2\right) \gtrsim T_i(t)\gap_i^2 - \log(1 / \rho_i)\log(T_i(t)).
\end{align*} 
\label{lemma:BettingGrowthRate}
\end{lemma}
\begin{proof}
Recall that \(\lambda_t = (\widehat{\mu}_{i, t-1} / 2)_+ \). Then, we derive the following asymptotic lower bound:
\begin{align*}
   \sum\limits_{j = 1}^{T_i(t)}\lambda_{i, t_i(j)}\gap_i - \lambda_{i, t_i(j)}^2 &= \frac{1}{4}\sum\limits_{j = 1}^{T_i(t)} 2\widehat{\mu}_{i, t - 1}\gap_i - \widehat{\mu}_{i, j - 1}^2\\
   &\geq \frac{1}{4}\sum\limits_{j = 1}^{T_i(t)} \gap_i^2 - (\gap_i - \widehat{\mu}_{i, j - 1})^2\\
   &\geq \frac{1}{4}\sum\limits_{j = 1}^{T_i(t)} \gap_i^2 - \varphi(T_i(j - 1), \rho_i)^2 \tag*{def. of \(\rho_i\)}\\
   &\geq\frac{1}{4}\sum\limits_{j = 1}^{T_i(t)} \gap_i^2 - \frac{4 \log(\log_2(2j) / \rho_i)}{j} \tag*{upper bound from \Cref{fact:LILBound}}\\
   &\gtrsim  \gap_i^2T_i(t) - \log\left(\frac{1}{\rho_i}\right)\log(T_i(t)), 
\end{align*} where the last line is because \(\sum_{j = 1}^{T_i(T)}1 / j \approx \log T_i(t)\). Thus, we have arrived our desired result.
\end{proof}

We now have the ingredients to present a proof of \Cref{thm:BettingSampleComplexity}.

\begin{proof}[Proof of \Cref{thm:BettingSampleComplexity}]
Combining the lower bound in \eqref{eqn:BettingGrowth} with \Cref{lemma:BettingGrowthRate}, we get the following asymptotic lower bound:
\begin{align*}
E_{i, t} \gtrsim \exp(T_i(t)\gap_i^2 - \log(1 / \rho_i') - \log(1 / \rho_i)\log(T_i(t)).
\end{align*}

Inverting the expression above, we get that the following lower bound sample complexity of a single arm,
\begin{align*}
T_i(t) \gtrsim \gap_i^{-2}\log(\gap_i^{-2})\log(1 / \rho_i) + \gap_i^{-2}\log(1 / \rho_i') +  \gap_i^{-2}\log(\varepsilon),
\end{align*} implies \(E_{i, t} \geq \varepsilon\) for \(\varepsilon > 0\). 

We can now derive a bound for \(\sum\limits_{t = 1}^\infty \ind{I_t \in \hypset_1, \captureset \not\subseteq \rejset_t}\).
\begin{align*}
    \sum\limits_{t = 1}^\infty \ind{I_t \in \hypset_1, \captureset \not\subseteq \rejset_t} &\leq \max_{\pi}\sum\limits_{i \in \hypset_1} \gap_i^{-2}\log(\gap_i^{-2})\log(1 / \rho_i) + \gap_i^{-2}\log(1 / \rho_i') +  \gap_i^{-2}\log(k / \pi(i))
\end{align*} where \(\pi\) is a mapping from \([|\hypset_1|]\) to \(\hypset_1\).

We get the following total bound:

\begin{align*}
    \sum\limits_{t = 1}^\infty \ind{\captureset \not\subseteq \rejset_t} =& \sum\limits_{t = 1}^\infty \ind{I_t \in \hypset_0, \captureset \not\subseteq \rejset_t} + \sum\limits_{t = 1}^\infty \ind{I_t \in \hypset_1, \captureset \not\subseteq \rejset_t}\\
    =& \sum\limits_{t = 1}^\infty \ind{I_t \in \hypset_0, \captureset \not\subseteq \rejset_t} + \sum\limits_{t = 1}^\infty \ind{I_t \in \hypset_1, \captureset \not\subseteq \rejset_t}\\
    \lesssim&\sum\limits_{i \in \hypset_0} \gap_i^{-2}\log(\log(\gap_i^{-2}) / \delta\rho_i) \\
    &+ \max_\pi \sum\limits_{i \in \hypset_1} \gap_i^{-2}\log(\gap_i^{-2})\log(1 / \rho_i) + \gap_i^{-2}\log(1 / \rho_i') +  \gap_i^{-2}\log(k / \pi(i)), 
\end{align*} where the asymptotic inequality is by \Cref{lemma:NullSampleComplexity}.

We know that we can apply \Cref{lemma:SuperuniformConcentration} at level \(\beta = \delta / 2\) to \(\rho_i\) for \(i \in [k]\) and \(\rho_i'\) for \(i \in \hypset_1\). Thus, the following happens with at least \(1 - \delta\) probability:
\begin{align*}
    \sum\limits_{t = 1}^\infty \ind{\captureset \not\subseteq \rejset_t}  \lesssim&   \sum\limits_{i \in \hypset_0} \gap_i^{-2}\log(\log(\gap_i^{-2}) / \delta) \\
    &+ \max_\pi\sum\limits_{i \in \hypset_1} \gap_i^{-2}\log(\gap_i^{-2})\log(1 / \delta) + \gap_i^{-2}\log(k / \pi(i)).
\end{align*}

Similar to the \Cref{thm:DiscreteESampleComplexity}, we can show two different bounds. The first is the following:
\begin{align*}
    \sum\limits_{t = 1}^\infty \ind{\captureset \not\subseteq \rejset_t}
    &\lesssim |\hypset_0|\gap^{-2}\log(\log(\gap^{-2}) / \delta) + |\hypset_1| \gap^{-2}\log(\gap^{-2})\log(1 / \delta), 
\end{align*} because \(\sum\limits_{i = 1}^{|\hypset_1|} \log(k / i) \leq k\). The second follows from dropping \(\pi(i)\):
\begin{align*}
    \sum\limits_{t = 1}^\infty \ind{\captureset \not\subseteq \rejset_t}
    \lesssim& \sum\limits_{i \in \hypset_0} \gap_i^{-2}\log(\log(\gap_i^{-2}) / \delta) \\
    &+ \sum\limits_{i \in \hypset_1} \gap_i^{-2}\log(\gap_i^{-2})\log(1 / \delta) + \gap_i^{-2}\log k.
\end{align*} Thus, we have shown both of our desired bounds.
\end{proof}

\section{Testing the average conditional mean}
\label{sec:ACE}

For simplicity, we will discuss results and proofs in this section under the single arm bandit case, so we will drop the arm index \(i\) when labeling terms. Our conclusions, however, do generalize to the general multi-arm bandit case. 

In \Cref{sec:SubGaussian} and \JJ, the null hypothesis for each arm we are concerned with is
\begin{align}
\text{``}\expect[X_t \mid \filtration_{t - 1}] \leq \mu_0 \text{ for all }t \in \naturals\text{ almost surely.''} \tag{H1}
\label{eqn:NullHyp}
\end{align} In the aforementioned settings, there is an additional assumption that \(X_{i, t}\) are i.i.d.\  across \(t \in \naturals\). Thus, \(\expect[X_t \mid \filtration_{t -1}]\) simply becomes \(\expect[X_t]\). We can also test a more general hypothesis of whether the means of \(X_t\) are less than or equal to \(\mu_0\) on average.

To formally define a notion of ``average mean'', let us consider the case where there is a single arm i.e.\ we have a sequence of rewards \(X_1, X_2, \dots\), where the \textit{average conditional mean} is defined as
\begin{align*}
    \overline{\mu}_t \equiv \frac{1}{t}\sum\limits_{j = 1}^t \expect[X_j \mid \filtration_{j - 1}].
\end{align*} 

Consequently, we can define a null hypothesis w.r.t.\ \(\overline{\mu}_t\):
\begin{align}
    \text{``}\overline{\mu}_t \leq \mu_0\text{ for all }t \in \naturals \text{ almost surely.''} \tag{H2}
    \label{eqn:CondNullHyp}
\end{align}

In the specific case where \(X_t\) are i.i.d.\ across \(t \in \naturals\), each with mean \(\mu\), then \(\expect[X_t \mid \filtration_{j - 1}] = \expect[X_t] = \mu\) for all \(t \in \naturals\), and \(\overline{\mu}_t = \mu\). Consequently, there would be no difference between testing the average conditional mean and testing the marginal mean, \(\mu\), because they are the same value. However, when  the distribution of \(X_t\) are not necessarily i.i.d.\ across \(t \in \naturals\), we will emphasize that not all valid tests for \eqref{eqn:NullHyp} are also valid for \eqref{eqn:CondNullHyp}. Generally, \eqref{eqn:NullHyp} is a ``stronger'' hypothesis than \eqref{eqn:CondNullHyp} in the sense that any distribution over \(X_t\) for \(t \in \naturals\) that satisfies \eqref{eqn:NullHyp} also satisfies \eqref{eqn:CondNullHyp}.

The difference between \eqref{eqn:NullHyp} and \eqref{eqn:CondNullHyp} is reflected in the fact that e-processes are supermartingales in \eqref{eqn:NullHyp}, but only upper bounded by a martingale in \eqref{eqn:CondNullHyp}.
\begin{proposition}
Assume that conditional distribution of \(X_t \mid \filtration_{t - 1}\) is always \(1\)-sub-Gaussian for all \(t \in \naturals\).
Consider a process of the form,
\begin{align*}
    E_{t} \coloneqq \sum\limits_{\ell = 1}^m w_{\ell} \exp\left(\sum\limits_{j = 1}^t \lambda_{\ell}(X_j - \mu_0) - \frac{\lambda_{\ell}^2}{2}\right)
\end{align*} where \(m \in \naturals \cup \{\infty\}\), \(\sum\limits_{\ell = 1}^m w_{\ell} \leq 1\). Under both \eqref{eqn:NullHyp} and \eqref{eqn:CondNullHyp}, \((M_t)\) is a e-process. Specifically, \((M_t)\) is 
\begin{enumerate}[label = (\roman*)]
    \item a \underline{nonnegative supermartingale} under \eqref{eqn:NullHyp}. \label{item:NSM}
    \item \underline{upper bounded by a nonnegative supermartingale} under \eqref{eqn:CondNullHyp}. \label{item:UBNSM}
\end{enumerate}
\label{prop:DiffNullHypotheses}
\end{proposition}
\begin{proof}
\ref{item:NSM} follows from \Cref{prop:DMIsNM}.

Without loss of generality, we will consider the case where \(m = 1\) and \(w_1 = 1\), since a convex combination of supermartingales is a supermartingale. Thus, \begin{align*}
E_t = \exp\left(\sum\limits_{j = 1}^t \lambda(X_j - \mu_0) - \frac{\lambda^2}{2}\right).
\end{align*}

 To prove \ref{item:UBNSM}, we first notice we can define a process \(M_t'\) that upper bounds \(E_t\):
\begin{align*}
    M_t' &= \exp\left(\sum\limits_{j = 1}^t \lambda(X_j - \expect[X_j \mid \filtration_{j - 1}]) - \frac{\lambda^2}{2}\right)\\
    & = \exp\left(\lambda\left(\sum\limits_{j = 1}^t X_j - \sum\limits_{j = 1}^t \expect[X_j \mid \filtration_{j - 1}]\right) - \frac{t\lambda^2}{2}\right)\\
    & = \exp\left(\lambda\left(\sum\limits_{j = 1}^t X_j - t\overline{\mu}_t\right) - \frac{t\lambda^2}{2}\right)\\
    &\geq \exp\left(\lambda\left(\sum\limits_{j = 1}^t X_j - t\mu_0\right) - \frac{t\lambda^2}{2}\right)\\
    &=E_t.
\end{align*} Now, we will show that \((M_t')\) is a supermartingale.
\begin{align*}
    \expect\left[\exp\left(\sum\limits_{j = 1}^t \lambda(X_j - \expect[X_j \mid \filtration _{j - 1}]) - \frac{\lambda^2}{2}\right) \mid \filtration_{t - 1}\right] &= \expect\left[\exp\left(\lambda (X_t - \expect[X_t \mid \filtration_{t - 1}]) - \frac{\lambda^2}{2}\right) \mid \filtration_{t - 1}\right]M'_{t - 1}\\
    & \leq M_{t - 1}',
\end{align*} where the last inequality is because the conditional distribution of  \(X_t\mid \filtration_{t - 1}\) is \(1\)-sub-Gaussian. Thus, we have shown both parts of our desired result.
\end{proof}

The \(E_t\) specified in \Cref{prop:DiffNullHypotheses} is an e-process, but not necessarily a nonnegative supermartingale, for any distribution under \eqref{eqn:CondNullHyp} where there exists a \(t \in \naturals\) such that \(\expect[X_t \mid \filtration_{t - 1}] > \mu_0\). Thus, the distinction highlighted in \Cref{prop:DiffNullHypotheses} is not vacuous.

Further, we will also note the following negative result that there exists processes that are e-processes under \eqref{eqn:NullHyp} but are not under \eqref{eqn:CondNullHyp}.

\begin{proposition}
Assume that conditional distribution of \(X_t \mid \filtration_{t - 1}\) is always \(1\)-sub-Gaussian for all \(t \in \naturals\). \((\pmh_t)\) is an e-process under all distributions satisfying \eqref{eqn:NullHyp}, but there exist \((\lambda_t)\) such that \((\pmh_t)\) \underline{is not an e-process} under all distributions that satisfy \eqref{eqn:CondNullHyp}.
\end{proposition}
\begin{proof}
\((\pmh_t)\) is an e-process under \eqref{eqn:NullHyp} by a similar argument to the proof of \Cref{prop:SafeBettingVariable}, since the conditional distribution of \(X_t \mid \filtration_{t - 1}\) is \(1\)-sub-Gaussian. However, we can provide a simple counterexample choice of \((\lambda_t)\) and distribution that satisfies \eqref{eqn:CondNullHyp} which cannot have expectation greater than \(1\) at a time \(t \in \naturals\). Let \(\mu_0 = 0\),  \(X_t = -1\) if \(t\) is odd, and \(X_t = 1\) if \(t\) is even.  Consider a \((\lambda_t)\) where \(\lambda_t = 0\) when \(t\) is odd and \(\lambda_t = 1\) when \(t\) is even. Then, 
\begin{align*}
    \expect[\pmh_t] = \exp(\lceil t / 2 \rceil).
\end{align*} Consequently, \((\pmh_t)\) with this choice of \((\lambda_t)\) is not an e-process under \eqref{eqn:CondNullHyp}, and we have proved our desired result.
\end{proof}

Thus, using adaptive strategies for selecting \((\lambda_t)\) like in \citet{waudby-smith_estimating_means_2021} for testing \eqref{eqn:CondNullHyp} is not necessarily straightforward, while mixture strategies in the form specified in \Cref{prop:DiffNullHypotheses} are valid e-processes for testing both \eqref{eqn:NullHyp} and \eqref{eqn:CondNullHyp}.

\end{document}